\documentclass[3p]{elsarticle}

\usepackage{amsmath,amssymb,amstext}
\usepackage{tabularx}
\usepackage{bm}
\usepackage{mathrsfs}
\usepackage{fancyhdr}
\usepackage{booktabs}
\usepackage{longtable}
\usepackage{multirow}
\usepackage{float}
\usepackage{braket}
\usepackage{amsthm}
\usepackage{bbold}
\usepackage{comment}
\usepackage[normalem]{ulem}

\usepackage[pdftex,pagebackref=false]{hyperref}
\hypersetup{
    unicode=false,          
    pdffitwindow=false,     
    pdfstartview={FitH},    
    pdfnewwindow=true,      
    colorlinks=true,        
    linkcolor=blue,         
    citecolor=green,        
    filecolor=magenta,      
    urlcolor=cyan           
}

\allowdisplaybreaks

\newtheorem{example}{Example}
\newtheorem{remark}[example]{Remark}
\newtheorem{definition}[example]{Definition}
\newtheorem{proposition}[example]{Proposition}

\newtheorem*{proposition*}{Proposition}

\newtheorem{innerbias}{Inductive Bias}
\newenvironment{bias}[1]
  {\renewcommand\theinnerbias{#1}\innerbias}
  {\endinnerbias}

\begin{document}

\title{Multi-Excitation Projective Simulation with a Many-Body Physics-Inspired Inductive Bias}

\author[1]{Philip A. LeMaitre\corref{cor1}}
\ead{Philip.Lemaitre@uibk.ac.at}
\cortext[cor1]{Corresponding author}
\affiliation[1]{organization={University of Innsbruck, Institute for Theoretical Physics},
addressline={Technikerstrasse 21a},
postcode={A-6020},
city={Innsbruck},
country={Austria}}
\author[2]{Marius Krumm}
\ead{Marius.Krumm@uibk.ac.at}
\affiliation[2]{organization={University of Innsbruck, Institute for Theoretical Physics},
addressline={Technikerstrasse 21a},
postcode={A-6020},
city={Innsbruck},
country={Austria}}
\author[3]{Hans J. Briegel}
\ead{Hans.Briegel@uibk.ac.at}
\affiliation[3]{organization={University of Innsbruck, Institute for Theoretical Physics},
addressline={Technikerstrasse 21a},
postcode={A-6020},
city={Innsbruck},
country={Austria}}

\date{\today}

\begin{abstract}
With the impressive progress of deep learning, applications relying on machine learning are increasingly being integrated into daily life. However, most deep learning models have an opaque, oracle-like nature that makes it difficult to interpret and understand their decisions. This problem led to the development of the field known as \emph{eXplainable Artificial Intelligence} (XAI). One method in this field known as \emph{Projective Simulation} (PS) models a chain-of-thought as a random walk of a particle on a graph with vertices that have concepts attached to them. While this description has various benefits, including the possibility of quantization, it cannot be naturally used to model thoughts that combine several concepts simultaneously. To overcome this limitation, we introduce \emph{Multi-Excitation Projective Simulation} (MEPS), a generalization that considers a chain-of-thought to be a random walk of several particles on a hypergraph. A definition for a dynamic hypergraph is put forward to describe the agent's training history along with applications to AI and hypergraph visualization. An inductive bias inspired by the remarkably successful few-body interaction models used in quantum many-body physics is formalized for our classical MEPS framework and employed to tackle the exponential complexity associated with naive implementations of hypergraphs. We prove that our inductive bias reduces the complexity from exponential to polynomial, with the exponent representing the cutoff on the number of particles that can interact. We numerically apply our method to two toy model environments and a more complex scenario that models the diagnosis of a broken computer. These environments demonstrate the resource savings provided by an appropriate choice of the inductive bias, as well as showcasing aspects of interpretability. A quantum model for MEPS is also briefly outlined and some future directions for it are discussed.
\end{abstract}

\begin{keyword}
projective simulation\sep reinforcement learning\sep inductive bias\sep explainable AI\sep quantum-inspired
\end{keyword}

\maketitle

\section{Introduction} \label{Section:Introduction}

Deep learning has become a powerful numerical tool, with various applications all over science and technology~\cite{MLinPhysics, Medicine, DeepLearningApplications}. At the heart of this technological revolution are Artificial Neural Networks (ANN), parameterized function ans\"{a}tze trained via gradient descent methods to achieve an ideal input-output behavior on data~\cite{DeepLearningBook}. Despite the enormous success of ANNs, their complex and mostly problem-agnostic structure makes it difficult to understand their ``reasoning process'', essentially turning ANNs into oracles. This, along with other pertinent issues related to reliability and trustworthiness \cite{LiuJinRobust,Causal1,Causal2,Adversarial1,Adversarial2}, led to the development of the field known as \emph{eXplainable Artificial Intelligence} (XAI)~\cite{XAI1, XAI2, XAI3, XAI4, XAI5, XAI6, XAITrustworthy}.

One promising approach to XAI realizes that many conscious human decision-making processes take the form of a \emph{chain-of-thought}~\cite{LLM1, LLM2}. The framework of Projective Simulation (PS)~\cite{Briegel_Cuevas2012, Mautner_Makmal2015} combines a model of deliberation, based on episodic memory, with reinforcement learning ~\cite{RLbook}. It thereby extracts the essential components of chains-of-thought, and realizes that deliberation processes can be understood as a random walk of a single particle on a graph with vertices representing concepts or thoughts. Since its first proposal in \cite{Briegel_Cuevas2012}, PS has been successfully applied to many domains~\cite{Foraging, Bees, Robotics, Navigation, PSforQuantumExperiments}; with ~\cite{Robotics,PSforQuantumExperiments} involving the deployment of PS in real-world settings. Since PS is a special case of MEPS, all its previous applications also provide suitable application domains for MEPS. In most of these applications, vertices in the graph (referred to as \textit{clips}) have a more basic interpretation, e.g., as remembered percepts, or actions, or sensorimotoric memories more generally. 

However, the representation of chains-of-thought as simple paths in a graph is limited and cannot easily capture thoughts which are most naturally understood by taking their composite structure into account. A wide range of applications combines several concepts to arrive at new concepts. Examples include logical deductions, small arithmetic calculations, thoughts that compare the advantages and disadvantages of a potential decision, thoughts that take into account the results of early steps in the deliberation, etc. In basic PS, a single excitation/particle has to represent all the short-term information used by the agent for the current decision, not allowing it to disentangle the structure of the thoughts. Therefore, in this paper, we introduce \emph{Multi-Excitation Projective Simulation} (MEPS), an extension of PS to multiple particles/excitations. In this extension, the transition probabilities are allowed to depend on the full particle configuration, allowing MEPS to model composite thoughts. Also here, each vertex in the graph represents an elementary concept and an excitation on a vertex expresses whether this concept is currently relevant. However, now each currently relevant concept can be represented by a separate excitation, allowing for the memory structure to be directly represented in a more disentangled fashion. Mathematically, our random walk steps now map sets of vertices to sets of vertices, naturally leading to the mathematical notion of hypergraphs~\cite{GraphBook1, GraphBook2, Dai_Gao2023}.

A naive implementation of MEPS tends to exhibit a complexity exponential in the size of the semantic graph. The root of this exponential complexity is the fact that the size of the power set of the vertices scales exponentially with the number of vertices. Therefore, in this paper, we also present an inductive bias that reduces this complexity to a low-degree polynomial. Our inductive bias is a classical analogue of the typical structures found in many-body physics (MBP)~\cite{ManyBody1, ManyBody2}. In MBP, many if not most phenomena can be understood as arising from fundamental elementary interactions of only a handful of particles. In particular, the standard model of particle physics, our most fundamental description of nature so far, only has interactions of at most four elementary particles~\cite{QFT1, QFT2, QFT3}. 



In this paper, we use many-body physics and its few-body interactions as inspiration for formalizing an inductive bias in \textbf{classical} machine learning. We prove that our inductive bias reduces the number of trainable parameters and the complexity of one random walk step from exponential to polynomial in the number of graph vertices. The degree of the polynomial is given by the cutoff for how many particles are allowed to interact. Furthermore, to limit the lengths of the random walks, we introduce modifications of our inductive bias suitable for layered feed-forward hypergraphs.

We numerically apply our MEPS methodology and the inductive bias -- along with a baseline comparison to agents based on standard PS, standard Q-learning, and a tabular PS-inspired multi-layer Q-learning agent -- to three synthetic environments. The first environment is a toy model extending the \emph{Invasion Game} of \cite{Briegel_Cuevas2012}, which can be seen as a special case of \emph{contextual bandit} problems~\cite{Bandits1, Bandits2}. Here, we modify the Invasion Game to include irrelevant information, calling it the \emph{Invasion Game with Distraction}. Its simplistic nature makes it a well-suited example to discuss the impact of different choices of the inductive bias. The second environment is a modification of the first with more actions and a reward that incorporates deceptive strategies used by the attacker; we call it the \emph{Deceptive Invasion Game}. The final environment models the real-world diagnosis and repair process of a broken computer, which we call the \emph{Computer Maintenance} environment. In this environment, we primarily showcase the interpretability aspects of MEPS agents, using the inductive bias to further illustrate the advantages of reducing agent complexity (over standard PS in particular). For this purpose, we train multi-layered MEPS agents. In the intermediate layer, the MEPS agent hypothesizes about the causes behind observed symptoms of the malfunctioning computer before picking certain fixes that it can apply.  

Since the inductive bias still allows for cyclic hypergraphs, the random walk has an infinite worst-case running time. This motivates us to formulate modifications of our inductive bias and the corresponding hypergraphs for which low order polynomial runtimes are guaranteed. More specifically, we prove worst case upper bounds for the runtimes of the random walks for layered hypergraphs. These worst case paths are rarely encountered in practice, so we leave the formulation and formal analysis of these inductive biases for the appendix.

The paper is organized as follows: First, in Section~\ref{Section:RelWorks}, we explain how MEPS is situated within the broader scope of XAI and reinforcement learning literature. Then, in Section~\ref{Section:PS}, we describe Single-Excitation PS. In Section \ref{Section:MEPS}, we present and define our Multi-Excitation PS agent, along with a dynamic hypergraph to model the agent's training history in Subsection \ref{Section:DynamicHypergraphTrainingHistory}. In Section~\ref{Section:Bias}, we develop the formalization of our inductive bias, including the exponential reduction in trainable parameters. However, since the motivation involves the quantum physical description of many-body systems, we will delay the presentation of this motivation towards the end in Section~\ref{Section:Quantum}.Before, in Section~\ref{Section:Numerics}, the numerical experiments from the three learning scenarios we consider can be found in Section \ref{Section:Numerics}. Afterwards, as already stated, we present the quantum many-body physical motivation of our classical inductive bias in Section~\ref{Section:Motivation}. Based on this quantum physical foundation, we then propose approaches towards an actual quantum MEPS agent in Section \ref{Section:QuantumMEPS}. Finally, in Section \ref{Section:Conclusions}, we discuss our results and suggest some promising future directions for the MEPS framework.

In ~\ref{Appendix:BrokenCompTrainDetails}, we present more details on the numerical experiments. In ~\ref{Appendix:OtherInductiveBiases}, we present modifications of our inductive bias for layered hypergraphs which guarantee polynomial upper bounds on the number of steps in the random walk. Afterwards, in~\ref{Appendix:RelationHypergraphs}, we prove that our inductive bias agents are fully compatible with the definition of MEPS agents given in Section~\ref{Section:MEPS}. Then, in ~\ref{Appendix:Complexity}, we prove the aforementioned bounds for different variants of the inductive bias.

\section{Related Works}\label{Section:RelWorks}

The field of XAI~\cite{XAI1, XAI2, XAI3, XAI4, XAI5, XAI6} is dedicated to making AI systems more human-understandable. One can broadly distinguish three different approaches. Concerning the success of ANNs, \emph{model-agnostic} explanations seek to make the decisions of artificial neural networks more understandable without affecting the architecture or training of the model. A large focus of the research community is on \emph{feature importance} methods, such as Shapley values~\cite{Shapley}. These measure the strength with which a part of the input or of the model contributes to the decision. Another commonly mentioned method is given by \emph{Local Interpretable Model-agnostic Explanations} (LIME)~\cite{LIME}. These methods consider an environment of a data point, and within this environment they approximate the model with a much simpler model, e.g. a linear approximation. Another important class of explanation methods within the family of model-agnostic interpretations is given by so-called \emph{counterfactual explanations}. A typical use-case is the formulation of ``\emph{What if something were different?}'' questions to see how hypothetical changes would influence the model's decisions~\cite{CounterfactualMedicine}.

A second approach to XAI is specific to \emph{Large Language Models} (LLM). These models can communicate with humans in natural language, providing an inherent advantage concerning their understandability. Crucially, the term \emph{chain of thought} was popularized within this setting~\cite{LLM1, LLM2}. The basic idea is to explicitly prompt a LLM to provide a step-by-step reasoning process to answer a question formulated with natural language. This provides a human-understandable justification for an answer. Furthermore, it turned out that the accuracy of the final answer itself also increased. MEPS takes a very different approach to chains of thought, not relying on natural language and the associated expensive costs, in particular regarding big data. Instead, the idea is to design a problem-specific hypergraph in such a way that a random walk path helps the user understand how to solve a problem. The excitations then can be interpreted as the currently relevant hypergraph semantics, while a random walk step is a formalization of the thought transitioning to the next relevant semantics embedded into the hypergraph design.

The last broad class of approaches to XAI is given by \emph{mechanistic interpretability}~\cite{MechanisticInterpretability}. This class seeks to obtain a detailed understanding of the individual steps of decision making processes, e.g. by discovering analogous sub-circuits~\cite{AutomatedCircuitDiscovery}. The extreme form are models with \emph{inherent interpretability} in which the architecture is designed right from the start to allow humans to understand its decision making process. Commonly mentioned examples are linear models~\cite{LIME} and decision trees~\cite{DecisionTrees}. \emph{Causal models}~\cite{Causal1, Causal2} intrinsically model environments and agents using cause-effect relations of observables, and promise to overcome harmful spurious correlations in data. The field of \emph{NeuroSymbolic AI} (NeSy)~\cite{NeSy1, NeSy2} works towards combining the rigorous explanation power of mathematical logic with the efficiency of ANN methods such as gradient descent. At last, also PS and our MEPS belong into the category of intrinsically interpretable models. 



Beyond the endeavour to use explainable AI to increase human understanding of the application domain, XAI is commonly mentioned together with the goal of trustworthiness~\cite{TrustResponsible, XAITrustworthy}. As AI becomes more integrated into human activities involving decision-making that impacts the lives of other people, it is important that we can trust these systems to make ethically responsible choices we would normally entrust to humans -- even more so since many of these systems are opaque at their foundation. Constructing transparent agent architectures should be the first step towards helping human experts develop machine learning models which are inherently more trustworthy. It is the purpose of this work to introduce such an architecture in the form of MEPS, leaving the philosophical considerations associated with the notion of trustworthiness to other experts for the time being \cite{Medicine,XAITrustworthy,TrustResponsible,SafeReinforcementLearning,XAI5,XAI6}.  


Both PS and the MEPS methodology introduced in this work are qualitatively very different from ANNs. The discrete nature of the (hyper)graph, the random walks and the update rule makes MEPS a \emph{tabular method}. A standard example for tabular methods are \emph{Q-learning} agents~\cite{RLbook}. Its tabular nature means that MEPS inherits the associated scalability problems also afflicting Q-learning. However, Q-learning only provides a look-up table for observation-action pairs, but does not provide a refined decision-making process involving intermediate steps. In particular, tabular Q-learning only provides interpretations in the form of a lookup table for reflex agents. The fact that MEPS is \textbf{not} a reflex agent is reflected in the fact that MEPS agents can be modelled to have several tabular layers, instead of just one, and that a decision-making process takes several random walk steps. In particular, such a refined layered structure gives several avenues for counterfactual interpretations specific to MEPS. The sampling in intermediate layers means that the agent decides that certain atomic clips (and their associated semantics) are relevant for the current deliberation, when it could have picked other atomic clips. One can even counterfactually alter the decision-making process by manually adding or removing excitations during a deliberation. Note that Q-learning can also be extended beyond the reflex agent paradigm by using PS as an inspiration to define a notion of multi-layered Q-learning agents (see Subsection \ref{Subsection:ComputerMaintenance} for details).

Besides explainability, another key topic in our paper is given by inductive biases. In machine learning, the term \emph{inductive bias}~\cite{InductiveBias1, InductiveBias2, InductiveBias3} refers to restrictions or modelling assumptions imposed on the trainable models before the training starts. These restrictions can be formalizations of domain knowledge about the problem or the solution. A common example is \emph{Convolutional Neural Networks} (CNN) \cite{CNN1, CNN2}, which assume translation-equivariance. The restrictions can also serve the purpose of making the model easier to interpret (for example by the use of modularity~\cite{Modularity1, Modularity2}), or making it more robust to out-of-distribution data (for example by integrating causal modeling~\cite{Causal1, Causal2}). However, our inductive biases are qualitatively very different, because they are tailored to MEPS rather than ANNs.

\section{(Single-Excitation) Projective Simulation} \label{Section:PS}
PS~\cite{Briegel_Cuevas2012, Mautner_Makmal2015} is a machine learning approach that models the basic process of how a chain of thought emerges as a random walk. The core idea is that each new thought is sampled from a probability distribution conditioned on the current thought. 

To formalize this idea, PS uses a so-called \emph{Episodic and Compositional Memory} (ECM). This ECM is modelled as a weighted, directed graph $G = (V,E,h)$. The vertices $c \in V$ are called \emph{clips} and we assume a labelling $V = \{ c_1, \ldots , c_{|V|} \}$. These clips have semantics attached to them: they might represent memories, elementary concepts, or other forms of thoughts. An example is shown in Figure \ref{Figure:PS}. To model a decision-making process, also called \emph{deliberation}, the agent performs a random walk over $V$. 

\begin{figure}[h!]
    \centering
    \includegraphics[width = 0.6 \linewidth]{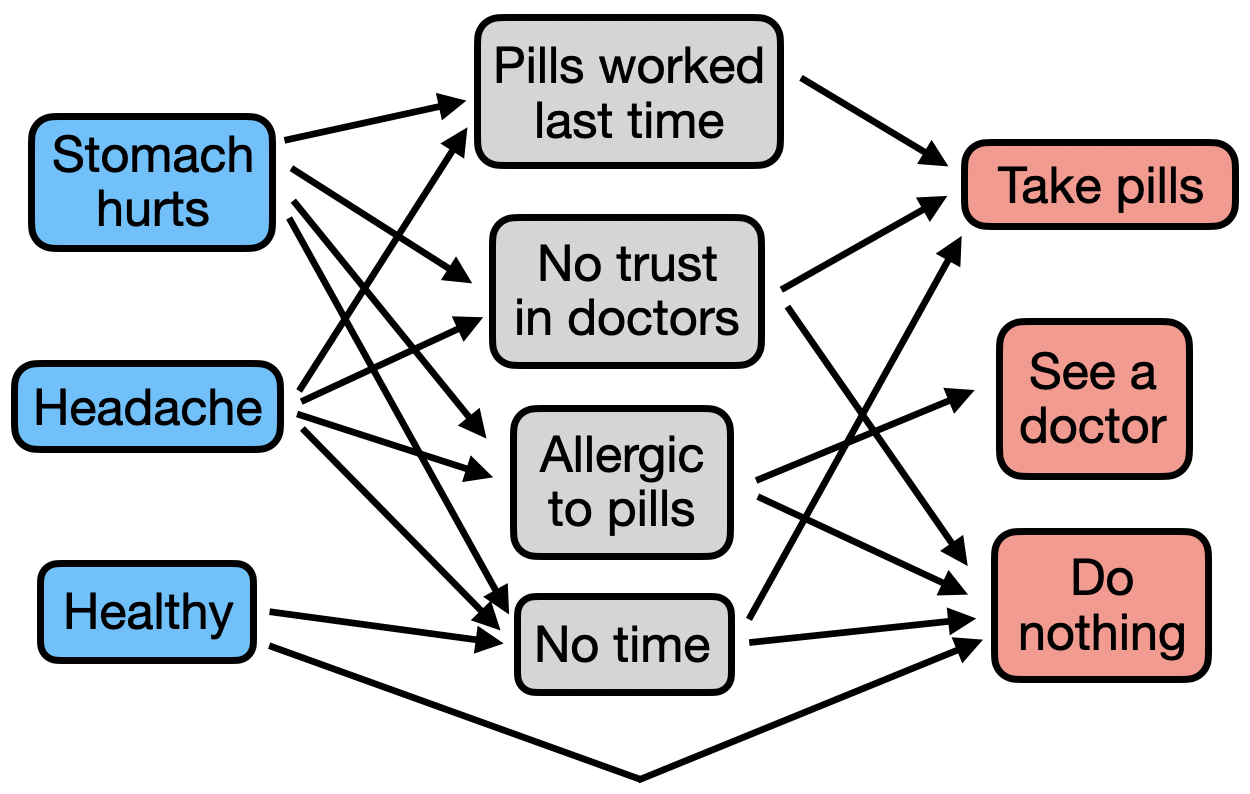}
    \caption{An example for the ECM of a PS agent contemplating how to deal with a small ailment. Observations/percepts are shown in blue and actions in red. Furthermore, there are grey internal clips representing intermediate thoughts that lead to a decision.}
    \label{Figure:PS}
\end{figure}

The edges $e \in E$ represent allowed transitions between clips, and since the ECM is directed, we will often write edges $e = (c_j, c_k)$ as $c_j \rightarrow c_k$. To each edge $e = c_j \rightarrow c_k$ at time step $n$, we assign a weight $h^{(n)}(e) \equiv h^{(n)}(c_j, c_k) \in \mathbb R$ that we call an \emph{h-value}; these serve as the trainable parameters of the agent.

Given a clip $c_j$, to sample the next clip, one considers the transition probability $p^{(n)}(c_k| c_j)$ constructed from all the $h$-values $h^{(n)}(c_j, c_k)$ as defined in \cite{Briegel_Cuevas2012}: 
\begin{align}
    p^{(n)}(c_k | c_j) := \frac{h^{(n)}(c_j, c_k)}{\sum_{m} h^{(n)}(c_j, c_m)}
\end{align}
We denote the above as the \emph{standard probability rule}. Another popular probability assignment, and one which is used heavily in this work, is the use of the softmax function, i.e. replace $h^{(n)}(c_j, c_k)$ with $e^{\beta h^{(n)}(c_j, c_k)}$ for some hyperparameter $\beta \in \mathbb R$. 

PS is usually applied within the \emph{Markov Decision Process} (MDP) setting of reinforcement learning~\cite{RLbook}. This means the agent interacts with an environment, where this interaction consists of discrete time steps that are each comprised of the following parts: at the beginning of the step, the agent obtains an \emph{observation} that it must respond to with an \emph{action} and then obtains a \emph{reward} $R\in \mathbb R$.

During the design of a PS agent, one has to decide how to ``couple in'' observations and ``couple out'' actions. In PS, observations are also called \emph{percepts}. The most popular approach assumes discrete finite observations and assigns a separate \emph{percept clip} to each percept (shown in blue in Fig.~\ref{Figure:PS}). Similarly, each of finitely many actions gets a separate \emph{action clip} in the ECM (shown in red in Fig.~\ref{Figure:PS}).

To train a PS agent in the setting of MDPs, the \emph{standard PS update rule} reinforces the entire deliberation path from percept to action. More specifically, after taking an action and receiving a reward $R^{(n)}$, each edge $c_j \rightarrow c_k$ is updated according to the following rule:
\begin{align}
    h^{(n+1)}(c_j, c_k) &= h^{(n)}(c_j, c_k) - \gamma (h^{(n)}(c_j, c_k) - h_{\mathrm{init}}) \\
    &\qquad \qquad \qquad + R^{(n)} g^{(n)}(c_j, c_k) \nonumber
\end{align}
If the standard probability assignment is used, we clamp the h-values to be no smaller than some base value (a hyper-parameter) $h_{\mathrm{min}} \ge 0$. Furthermore, $\gamma \in [0,1]$ is the \emph{forgetting} hyperparameter that controls how fast an h-value decays back to its initial value $h_{\mathrm{init}}$. This forgetting mechanism mitigates overfitting, acts as a soft regularization, and allows for faster adaptation to shifts in the transition function of the environment. $g^{(n)}(c_j, c_k)$ is the \emph{glow} factor that allows for handling of delayed rewards and is defined via
\begin{align}
    g^{(n)}(c_j, c_k) = \begin{cases}
        1 & \text{ if } c_j \rightarrow c_k \text{ on}\\ & \text{ last path }\\
        (1-\eta) g^{(n-1)}(c_j, c_k) & \text{ else} 
    \end{cases}
\end{align}
and initialized to $0$. The \emph{glow dampening factor} $\eta \in [0,1]$ is a hyperparameter, and plays a role similar to the discount factor in returns and value functions~\cite{OptimalPS}. The standard PS update rule can be interpreted as a form of Hebb's learning rule ``\emph{What fires together wires together}''.

\section{Multi-Excitation PS} \label{Section:MEPS}

While PS models chains of thought as random walks it cannot naturally represent reasoning steps that have a composite structure. For example, the decision to eat in a restaurant might depend both on the financial situation of the agent as well as their appetite. In PS, the current clip has to store all the short-term information the agent considers in the deliberation. Therefore, clips need to carry the semantics of all relevant observables, such as $c = ($hungry, $\ge 100$ USD, no time to cook, good restaurant nearby$)$.  

For the purpose of interpretability, it is important to explicitly represent different observables and degrees of freedom. To make this possible, we first imagine the random walk of PS as an excitation or a particle moving along the ECM. We will use the terms \emph{particle} and \emph{excitation} interchangeably. Now, to be capable of explicitly representing different observables as separate entities, we replace the single excitation with multiple excitations. With this change, it is also possible to have one excitation for each value of each observable.

\begin{figure}[h!]
    \centering
    \includegraphics[width = 0.7 \linewidth]{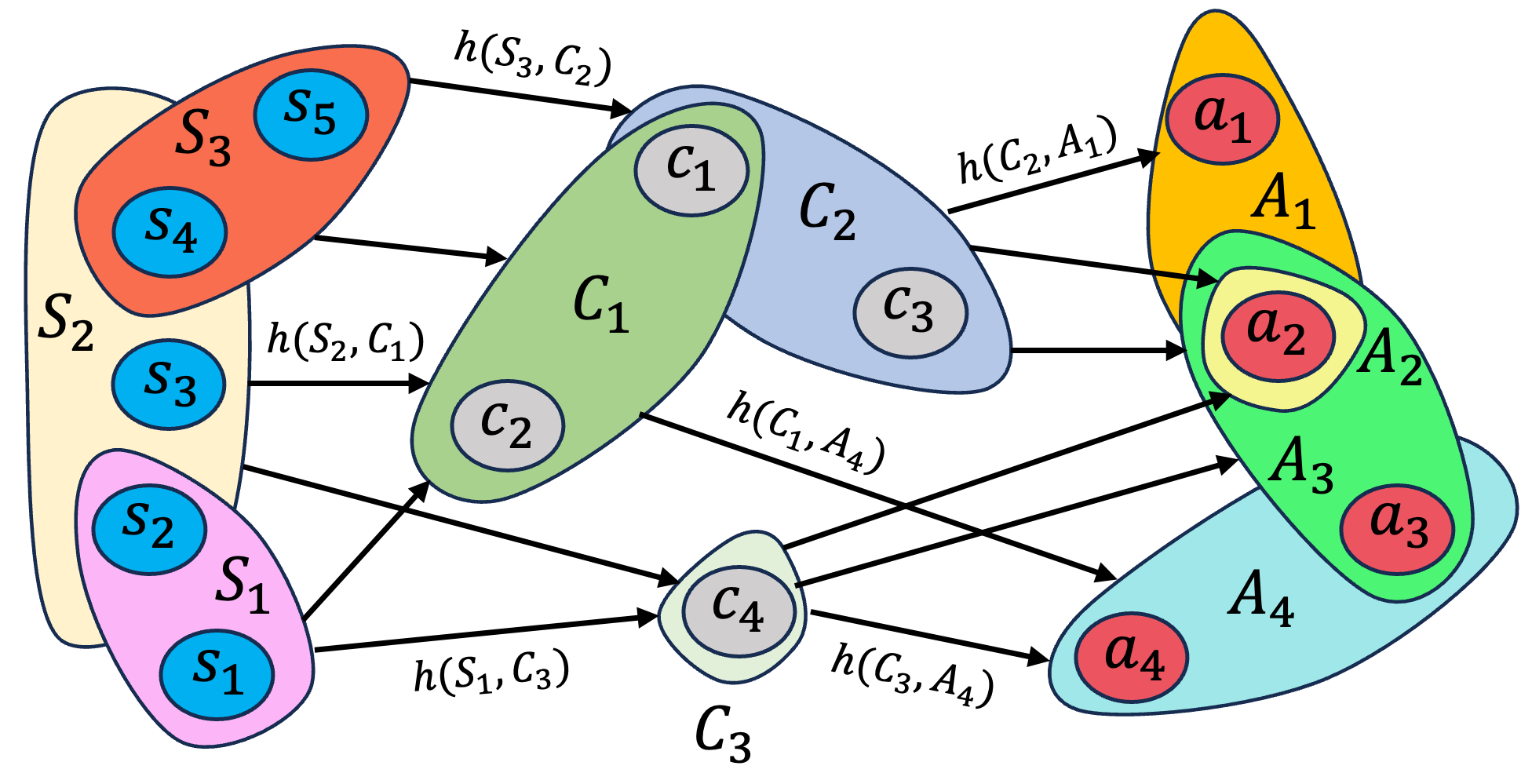}
    \caption{An example of a directed, weighted hypergraph describing the ECM of a typical MEPS agent in a reinforcement learning setting. Atomic percept clips are represented in blue with a lowercase $s$, atomic intermediate clips in grey with a lowercase $c$, and atomic action clips in red with a lowercase $a$; the domains and codomains of hyperedges are labelled with capital letters whose clip type corresponds to their lowercase version. Each hyperedge $e \in E$ also has an h-value $h(e)$ associated with it.}
    \label{Figure:Hypergraph}
\end{figure}

In the dinner example from above, one could have an ECM with $V = \{$full, hungry, $\ge 100$ USD, $< 100$ USD, no time to cook, plenty of time, good restaurant nearby, good restaurants far away$\}$. Now, the current short-term memory of the agent can be described by an \emph{excitation configuration} such as $C_{\mathrm{now}} = \{ $hungry, $\ge 100$ USD, no time to cook, good restaurant nearby$\} \subset V$, graphically represented as putting an excitation on each of the atomic clips in $C_{\mathrm{now}}$. The edges of PS are replaced with objects that move from a current excitation configuration to the next excitation configuration. Mathematically, this can be formalized using hypergraphs~\cite{GraphBook1, GraphBook2, Dai_Gao2023}:

\begin{definition} \label{Definition:Hypergraph}
    A \emph{directed hypergraph} $G = (V, E)$ consists of a finite set $V$ and a set $E \subset \mathcal ( \mathcal P(V)\setminus \{ \varnothing \} ) \times (\mathcal P(V) \setminus \{ \varnothing \})$, with $\mathcal P(V)$ the power set of $V$. The elements of $V$ are referred to as \emph{vertices} or \emph{nodes} while the elements of $E$ are referred to as \emph{hyperedges}. For the sets of vertices  $ V_{\mathrm{in}} \equiv \{v_{j_1}, \ldots , v_{j_D}\}$ and $V_{\mathrm{out}} \equiv \{v_{k_1}, \ldots , v_{k_C}\}$ and hyperedge $e = (V_{\mathrm{in}}, V_{\mathrm{out}}) \in E$, we call $V_{\mathrm{in}}$ the \emph{domain} or \emph{tail} and $V_{\mathrm{out}}$ the \emph{codomain} or \emph{head} of the hyperedge $e \in E$. Hyperedges will also be referred to using the notation $V_{\mathrm{in}} \rightarrow V_{\mathrm{out}}$. 

    A \emph{weighted, directed hypergraph} $G = (V,E,h)$ is a hypergraph $G = (V,E)$ together with a weight function $h: E \rightarrow \mathbb R$.
\end{definition}

\begin{definition} \label{Definition:MEPSarchitecture}
    A \emph{standard Multi-Excitation Projective Simulation} (MEPS) agent is given by a weighted, directed hypergraph $G = (V, E, h)$ that we refer to as the \emph{Episodic and Compositional Memory}(ECM) of the agent (compare Figure \ref{Figure:Hypergraph}). We refer to the elements $c \in V$ as \emph{atomic clips} and use the notation $V = \{c_{1}, \ldots , c_N \}$. Subsets $C \subset V$ will be referred to as \emph{excitation configurations}. Furthermore, we will often use the short-hand notation $c_j \equiv j$, identifying atomic clips with their labels. 
\end{definition} 

\begin{remark}\label{Remark:TimeLable}
    Since the weight function $h$ represents our trainable parameters (specifically, the ordered list of \emph{h-values} $\Big(h(e_1), \dots, h(e_n)\Big)$ for $E=\{e_1,\dots ,e_n\}$), we will often update it. When we need to make clear that we refer to a specific time step $n$, we will use the notation $h^{(n)}$.
\end{remark}

Similarly to PS, we envision MEPS to be used in a reinforcement learning setting. This requires us to make choices about how percepts/observations are represented and about how actions are coupled out. For this purpose, we require that there are some fixed input and output coupling functions that connect the external behaviour of the agent with its internal model:

\begin{definition}
    Let $\mathrm{OUT} \subset P(V)\setminus \{ \varnothing \}$ and $\mathrm{IN} \subset P(V)\setminus \{ \varnothing \}$ denote output and input sets, respectively. In the setting of Markov Decision Processes, a MEPS agent is also equipped with the following two functions: an \emph{input coupling function} $\mathcal I: \text{Observations} \rightarrow \mathrm{IN}$ and an \emph{output coupling function} $\mathcal O: \mathrm{OUT} \rightarrow \text{Actions}$.

    Upon receiving an observation $\mathrm{obs}$, excitations are put on the atomic (percept) clips in $\mathcal I(\mathrm{obs})$ in the agent's ECM, triggering deliberation through the ECM (see Def. \ref{Definition:MEPSdynamics}) until reaching a set of atomic clips $C_{\mathrm{act}}$ contained in $\mathrm{OUT}$. Then, the action $\mathcal O(C_{\mathrm{act}})$ is used by the agent on the environment.
\end{definition}

 
With the previously established structure, we can now explain how a deliberation step of a MEPS agent works:

\begin{definition}\label{Definition:MEPSdynamics}
    Consider a standard MEPS agent and a current excitation configuration $C_{\mathrm{now}} = \{c_{m_1}, \dots, c_{m_x}\} \subset V$ with $m_j < m_k$ for $j< k$. The sampling of the next excitation configuration $C_{\mathrm{next}}$ is referred to as a \emph{random walk step} or \emph{deliberation step}. This step proceeds as follows:

    \begin{enumerate}
        \item Collect a (ordered) list 
        \begin{align*}
        \mathcal H_{\mathrm{relevant}} &= \Big( h(C_{\mathrm{now}}, C_{\mathrm{next}}) \Big| (C_{\mathrm{now}} \rightarrow C_{\mathrm{next}}) \in E \Big) \\ 
        &\equiv \Big( h( C_{\mathrm{now}}, \bullet)\Big).
        \end{align*}
        We refer to this list as the \emph{relevant $h$-values} for the next random walk step.
        
        \item Turn the list $\mathcal H_{\mathrm{relevant}}$ into a list of probabilities, e.g. by applying a softmax function or by using the \emph{standard probabilities} 
        \begin{align*}
        p(C_{\mathrm{next}}| C_{\mathrm{now}}) = \frac{h(C_{\mathrm{now}}, C_{\mathrm{next}})}{\sum_{(C_{\mathrm{now}} \rightarrow C') \in E} h(C_{\mathrm{now}}, C')}.
        \end{align*}

        \item Sample the next excitation configuration $C_{\mathrm{next}}$ using the probabilities from the previous step. 
    \end{enumerate}
    
\end{definition}

For learning, we directly adapt the standard PS update rule to our new concept of $h$-values.

\begin{definition}\label{Definition:MEPSlearning} 
Consider a MEPS agent with ECM $(V,E,h)$.

In addition, consider two further weight functions $g, h_{\mathrm{init}}: E \rightarrow \mathbb R$ for the directed hypergraph $(V,E)$, and two hyperparameters $\gamma \in [0,1]$ and $\eta \in [0,1]$. $h_{\mathrm{init}}$ gives the initialization of the \emph{h-values}, $g$ gives the \emph{glow-factors} or \emph{glows}, $\eta$ is the \emph{glow damping factor}, and $\gamma$ the forgetting factor. Before the first random walk, all glows are initialized to $0$.

Then, the \emph{standard MEPS update rule} proceeds as follows:

\begin{enumerate}
\item At the end of a random walk $ \mathcal R \equiv C_{j_1} \rightarrow \cdots \rightarrow C_{j_m}$ with $C_{j_k} \subset V \ \  \forall k$, for all $(C \rightarrow C') \in E$ the glow is updated according to:
\begin{align*}
    g^{(n)}(C, C') = 
    \begin{cases}
        1 & \text{ if } \exists k \text{ s.t. } C = C_{j_k}\\ & \text{ and } C' = C_{j_{k+1}}\\
        (1-\eta) g^{(n-1)}(C, C') & \text{ else} 
    \end{cases}
\end{align*}

\item The $h$-values for all hyperedges $(C \rightarrow C') \in E$ are then updated using the current reward $R^{(n)}$:
\begin{align*}
    h^{(n+1)}(C, C') = \ & h^{(n)}(C, C') - \gamma \cdot \Big(h^{(n)}(C, C')- h_{\mathrm{init}}\Big)  \\
    & + R^{(n)} g^{(n)}(C, C')
\end{align*}
If the standard probability assignment is used, we clamp the h-values to be no smaller than some hyper-parameter $h_{\mathrm{min}} \ge 0$.
\end{enumerate}

\end{definition}

\subsection*{The Training History of MEPS as a Dynamic Hypergraph}\label{Section:DynamicHypergraphTrainingHistory}

To rigorously formalize the training history of a standard MEPS agent, we propose the following definition for a dynamic hypergraph, which is a generalization of the dynamic graph definition found in \cite{Vehlow_Beck_Weiskopf2016} plus an additional modification:

\begin{definition} \label{Definition:DynamicHypergraph}
Let $T \subset \mathbb R$ be a parameter space with elements $t \in T$. Consider a weighted, directed hypergraph $G = (V, E, h)$ such that the vertex and hyperedge sets can be partitioned as $V = \bigcup_{t \in T} V_t$ and $E = \bigcup_{t \in T} E_t$, with $E_t \subset (\mathcal P(V_t) \setminus {\varnothing}) \times (\mathcal P(V_t) \setminus {\varnothing})$, respectively. A \emph{weighted, directed, dynamic hypergraph} $\mathcal G$ is a collection of sub-hypergraphs $\{G_t\}_{t\in T}$ where each $G_t=(V_t,E_t,h_t)$ is called a leaf of $\mathcal G$. Each leaf has a weight function $h_t: E_t \rightarrow \mathbb R$ which corresponds to a domain restriction of the weight function $h: E\rightarrow \mathbb R$. 

The initialization $h_{\text{init}}$ in particular is the weight function corresponding to the smallest leaf index $\min_{t \in T} t$, assuming a minimal index exists.
\end{definition}

Our definition allows us to explicitly relate all sub-hypergraphs appearing in the set through the identification of each of their weight functions $h_t$ as constant $t$ slices of the weight function $h$. In this way, the weight function $h$ acts as a sort of glue that stitches the sub-hypergraphs together, inducing a flow through the set. This is in stark contrast to the definition in \cite{Vehlow_Beck_Weiskopf2016} or other related definitions in the (hyper)graph visualization literature (to our knowledge) \cite{beck_taxonomy_2017}.

If we endow $h$ with the explicit form given in Definition \ref{Definition:MEPSlearning}, then the entire MEPS algorithm can also be viewed as a hypergraph generation tool, where hypergraphs with specific properties could be obtained after training by tailoring the agent architecture and update rule along with the learning environment. This process would produce a final hypergraph but if one also stores the generated hypergraphs at each training step, then the MEPS algorithm can also generate dynamic hypergraphs, which could subsequently be analyzed using standard (hyper)graph visualization techniques \cite{beck_taxonomy_2017,fischer_towards_2021}. Much infrastructure that is currently underdeveloped in normal PS implementations, such as the single-excitation PS graph surgery rules proposed in \cite{Briegel_Cuevas2012}, would need to first be developed before such a proposal could be fruitfully initiated. However, the inductive bias that will be described later in Section \ref{Section:Bias} may provide a temporary solution to the problem of generating a hypergraph with an arbitrary size, essentially by adaptively restricting the size of the subsets appearing in the edge set $E$ during training.

Because MEPS is an explainable model, the dynamic hypergraph inherits this explainability and one can also visualize how the meaning of the sub-hypergraphs evolves over time. We believe the latter is an interesting application of (hyper)graph visualization to machine-learning training histories and XAI in general \cite{fischer_visual_2021}.

As a technical aside, we require that a partition of the hypergraph $G$ can always be found such that each $h_t$ is well-defined. In many applications, the set of sub-hypergraphs can be interpreted as a time series, so that one can consider a larger hypergraph whose hyperedge and vertex sets are simply the union of all those that appear in the time interval. Then it is straightforward to construct the weighted, directed dynamic hypergraph. This is especially true for MEPS, which Definition \ref{Definition:DynamicHypergraph} was originally constructed for, as each successive sub-hypergraph after the initialization is generated upon applying the update rule in Definition \ref{Definition:MEPSlearning} such that the partitioning is guaranteed.

We also believe Definition \ref{Definition:DynamicHypergraph} will be useful to machine-learning practitioners, specifically those using hypergraph learning methods \cite{yadati_hypergcn_2019} or hypergraph neural networks \cite{feng_hypergraph_2019}, as a way to talk about agent learning/training history.



\section{A Physics-Inspired Inductive Bias} 
\label{Section:Bias}

In this section, we will present our inductive bias that will allow us to significantly reduce the computational complexity of MEPS agents. 
Since the detailed motivation is inspired by quantum many-body physics to formulate classical analogues, we postpone a detailed explanation of the motivation to Section~\ref{Section:Motivation}. Nonetheless, we emphasize that we have written said section in an accessible way which we believe to be digestible also for non-physicists.

However, already here, we provide a basic, conceptual discussion motivating the main decisions for our inductive bias. The main cause for the exponential complexity of the hypergraph is that we assign transition probabilities or $h$-values $h(C_{\mathrm{in}} , C_{\mathrm{out}})$ for full excitation configurations $C_{\mathrm{in}/\mathrm{out}}$, in a hypergraph that can have exponentially many such configurations.

In many-body physics, one also investigates transition probabilities of many excitations. Such transition probabilities can usually be understood as arising from combinations of elementary interactions of only a handful of excitations. In general, these interactions can destroy or create excitations, a typical example being nuclear scattering.

This motivates us to introduce a finite set $IO$ which contains tuples $(i,o)$ of integers. For each $(i,o) \in IO$, there exists an elementary transition which converts $i$ ingoing excitations into $o$ outgoing excitations. Typically in the most well-known and successful physical models (e.g. the Ising model) $i,o \le 4$, so that only a small number of excitations are required to effectively capture a wide range of phenomena. In our classical analogy, we measure the ``strength'' or amplitude of this elementary transition with $h$-values $h_{(i,o)}(C^{(i)}_{\mathrm{in}} , C^{(o)}_{\mathrm{out}})$. This amplitude depends not only on the number of scattering excitations, but also on the precise state of the $i$ ingoing excitations $C^{(i)}_{\mathrm{in}}$ and the $o$ outgoing excitations $C^{(o)}_{\mathrm{out}}$. In our inductive bias, these \emph{many-body $h$-values} $h_{(i,o)}(C^{(i)}_{\mathrm{in}} , C^{(o)}_{\mathrm{out}})$ will be the trainable parameters.

Since these $h_{(i,o)}$ express elementary interactions/transitions, we allow for the presence of other excitations. This means that $C^{(i)}_{\mathrm{in}} \subset C_{\mathrm{in}}$ and $C^{(o)}_{\mathrm{out}} \subset C_{\mathrm{out}}$, i.e. the ingoing excitations and outgoing excitations of an elementary interaction are only subsets of the full excitation configurations now and later. However, physics does not allow arbitary transitions between particles. Conservation laws such as charge conservation and locality restrict the allowed transitions. We use sets $E^{(i,o)} \subset \mathcal P(V)\times \mathcal P(V)$ to define which elementary transitions we allow, i.e. only for $(C^{(i)}_{\mathrm{in}} , C^{(o)}_{\mathrm{out}}) \in E^{(i,o)}$ we assign (trainable) $h_{(i,o)}$.

In our classical analogy, we make the convention that one random walk step corresponds to one elementary transition occurring. Furthermore, as already stated, we measure the strength for each elementary process with the $h_{(i,o)}$. Therefore, to sample the next excitation configuration, we  just make a list of all the $h_{(i,o)}(C_{\mathrm{in}}^{(i)} , C_{\mathrm{out}}^{(o)})$ such that $C_{\mathrm{in}}^{(i)} \subset C_{\mathrm{in}}$, and directly sample a $C_{\mathrm{out}}^{(o)}$ from these $h_{(i,o)}$, using e.g. the standard probability assignment or a softmax. Then, the next full excitation configuration is simply given by $ C_{\mathrm{out}} := C^{(o)}_{\mathrm{out}} \cup ( C_{\mathrm{in}} \setminus C^{(i)}_{\mathrm{out}} )$.


Now, after this exposition, we give the full formal definition of our inductive bias, including its update rule:
\begin{bias}{1} \label{Bias:ManyBody}
    Given a non-empty finite set $V = \{c_1, \ldots c_{|V|}\}$, specify the following objects:
    \begin{enumerate}
    \item A finite set $IO \subset \mathbb N_{>0}^2$, where the elements $(i,o) \in IO$ are the allowed (pairs of) numbers of ingoing and outgoing excitations for which we will introduce \emph{many-body h-values} $h_{(i,o)}$.

    \item For all $(i,o) \in IO$, let $E^{(i,o)}_{\mathrm{all}}$ be the set of all $(C^{(i)}_{\mathrm{in}} , C^{(o)}_{\mathrm{out}}) \in (\mathcal P(V) \setminus \{ \varnothing \}) \times (\mathcal P(V) \setminus \{ \varnothing \})$ with $|C^{(i)}_{\mathrm{in}}| = i$ and $|C^{(o)}_{\mathrm{out}}| = o$, and $C^{(i)}_{\mathrm{in}} \ne C^{(o)}_{\mathrm{out}}$. Here, the last condition serves to rule out transitions that do nothing. Then, specify a subset $E^{(i,o)} \subset E^{(i,o)}_{\mathrm{all}}$ which serves to describe the set of allowed transitions for $(i,o)$. The notation $C^{(i)}_{\mathrm{in}} \rightarrow C^{(o)}_{\mathrm{out}}$ will also be used for $e = (C^{(i)}_{\mathrm{in}}, C^{(o)}_{\mathrm{out}}) \in E^{(i,o)}$.

    \item For each $(i,o) \in IO$, there is a (ordered) list 
    \begin{align}
        & H^{(i,o)} =  \Big\{ h_{(i,o)}\big(\{c_{j_1}, \ldots, c_{j_i} \},  \{c_{k_1}, \ldots, c_{k_o}\} \big)  \\ 
        &\Big| \big(\{c_{j_1}, \ldots, c_{j_i}\} \rightarrow \{c_{k_1}, \ldots, c_{k_o}\} \big) \in E^{(i,o)} \Big\}. \nonumber
    \end{align}
    The $h_{(i,o)}\big(\{c_{j_1}, \ldots, c_{j_i} \},  \{c_{k_1}, \ldots, c_{k_o}\} \big) \in \mathbb R$ are our trainable parameters and are called \emph{many-body h-values}.

    \item For each $(i,o) \in IO$, there is a (ordered) list $H^{(i,o)}_{\mathrm{init}}$ specifying the initialization of each element of $H^{(i,o)}$. Similarly, for each $(i,o)\in IO$, there is a (ordered) list $G^{(i,o)}$ storing the \emph{glow-factors} for all many-body h-values $h_{(i,o)}$.
    \end{enumerate}
    
    Given an excitation configuration $\{c_{m_1}, \ldots c_{m_x}\}$, a \emph{random walk step} or \emph{deliberation step} deciding the next excitation configuration proceeds as follows:

    \begin{enumerate}
        \item Collect a list $\mathcal H_{\mathrm{relevant}}$ of all many-body $h$-values $h_{(i,o)}\big(\{c_{j_1}, \ldots, c_{j_i} \}, \{c_{k_1}, \ldots, c_{k_o}\} \big) \in H^{(i,o)}$ with $(i,o) \in IO$ such that $\big(\{c_{j_1}, \ldots, c_{j_i} \} \rightarrow \{c_{k_1}, \ldots, c_{k_o}\} \big) \in E^{(i,o)}$ and $\{c_{j_1}, \ldots, c_{j_i}\} \subset \{ c_{m_1}, \ldots , c_{m_x}\}$. We refer to those as the \emph{relevant} many-body h-values.

        \item Turn $\mathcal H_{\mathrm{relevant}}$ into a list of probabilities by using the \emph{standard propabilities} (or using the softmax function for example)
        \begin{equation}
        \frac{h}{\sum_{\tilde h \in \mathcal H_{\mathrm{relevant}}} \tilde{h}}\,,
        \end{equation} 
        for each $h \in \mathcal H_{\mathrm{relevant}}$, then sample one transition $(\{c_{j_1}, \ldots, c_{j_i} \} \rightarrow  \{c_{k_1}, \ldots, c_{k_o} \}) \in \bigcup_{(i',o') \in IO} E^{(i',o')}$.

        \item In the original configuration $\{ c_{m_1}, \ldots c_{m_x}\}$, remove all excitations in $\{c_{j_1}, \ldots, c_{j_i} \}$, and put excitations into $\{c_{k_1}, \ldots, c_{k_o}\}$. If $ \big( \{c_{m_1}, \ldots , c_{m_x}\} \setminus \{c_{j_1}, \ldots, c_{j_i} \} \big) \cap \{c_{k_1}, \ldots, c_{k_o}\} \ne \varnothing $, we keep those excitations and discard the second excitations for those atomic clips (see Remark \ref{Remark:Occupied}).
    \end{enumerate}
    
    Consider now the $n$-th step in the episode, and write an explicit time label $^{(n)}$ on the $h$-values and glows. Upon receiving a reward $R^{(n)}$, the \emph{many-body MEPS update rule} updates the many-body h-values $h^{(n)}_{(i,o)} := h^{(n)}_{(i,o)}\big(C^{(i)}_{\mathrm{in}}, C^{(o)}_{\mathrm{out}}\big)$ for all $(i,o) \in IO$ and all $C^{(i)}_{\mathrm{in}} \rightarrow C^{(o)}_{\mathrm{out}} \in E^{(i,o)}$ according to the rule
    \begin{equation}
        h^{(n+1)}_{(i,o)} = h^{(n)}_{(i,o)}-\gamma \cdot (h^{(n)}_{(i,o)}-h^{(0)}_{(i,o)}) + R^{(n)} g^{(n)}_{(i,o)}\,,
    \end{equation}
    where $h^{(0)}_{(i,o)} \in H^{(i,o)}_{\mathrm{init}}$ is the initialization, $R^{(n)}$ is the reward of the current action, and $\gamma \in [0,1]$ is a fixed \emph{forgetting} hyperparameter. $g^{(n)}_{(i,o)} \equiv g^{(n)}_{(i,o)}(C^{(i)}_{\mathrm{in}}, C^{(o)}_{\mathrm{out}}) \in G^{(i,o)}$ is the \emph{glow}-factor, updated after each (full) random walk via
    \begin{align}
        g^{(n)}_{(i,o)} = \begin{cases} 1 &\text{ if } C^{(i)}_{\mathrm{in}} \rightarrow C^{(o)}_{\mathrm{out}} \text{ on the last path}\\
        (1-\eta) g^{(n-1)}_{(i,o)}, &\text{ else.}\end{cases} 
    \end{align}
    where $\eta \in [0,1]$ is the \emph{glow damping} hyperparameter. All glows are initialized to $0$.
    
    If the standard probability assignment is utilized, we clamp the many-body h-values after updates to be larger than some hyperparameter $h_{\mathrm{min}} \ge 0$.
\end{bias}

\begin{remark}
\label{Remark:Occupied}
It can happen that sampled transitions put excitations into atomic clips that are already occupied. In Inductive Bias \ref{Bias:ManyBody}, we made the choice that the atomic clip simply stays excited, i.e. it continues to carry exactly one excitation, effectively discarding the second excitation. 

We made this choice because we associate atomic clips with concepts or beliefs, and the excitation tells us whether the concept represented by the atomic clip is currently relevant.

However, as we will explain in more detail in Section \ref{Section:Quantum}, this choice cannot naturally be linked to the behavior of any elementary particle. In fact, it is an irreversible process: the excitation that jumps on an already excited atomic clip cannot jump back and is thereby annihilated.
\end{remark}

\begin{definition}
    The weighted, directed hypergraph obtained by using $E^{\mathrm{many-body}} := \bigcup_{(i,o) \in IO} E^{(i,o)}$ as the set of hyperedges and the many-body h-values $h_{(i,o)}$ as weights is called the \emph{many-body hypergraph}. 
\end{definition}

Now, we have two hypergraphs, the ECM and the many-body hypergraph. Similarly, we have the standard $h$-values $h$ and the many-body $h$-values $h_{(i,o)}$. While conceptually related, it is not obvious that the definitions in Inductive Bias \ref{Bias:ManyBody} are compatible with Definition \ref{Definition:MEPSarchitecture}. In \ref{Appendix:RelationHypergraphs}, we show that the definitions are indeed compatible during inference when using the standard probability assignment, by showing how to construct the standard $h$-values $h$ from the many-body $h$-values $h_{(i,o)}$. 

Furthermore, it is important to emphasize that $h$ and $h_{(i,o)}$ are only equivalent for inference, NOT during learning. Updating $h_{(i,o)}$ will also affect $h$ for transitions that did not occur in the random walk. However, for the rest of the paper, it is enough to only work with the $h_{(i,o)}$. 

In many scenarios, it will be natural to consider a layered ECM in which excitations move from layer to layer, similarly to feed-forward artificial neural networks:
\begin{definition} \label{Definition:Layers}
    A \emph{weighted, layered hypergraph} is a weighted, directed hypergraph $G = (V, E, h)$ together with a partition $L = (L_1,\ldots , L_D)$ of $V$ (i.e. $L_j \cap L_k = \varnothing \ \forall j\ne k$ and $L_j \ne \varnothing \ \forall j$ and $\bigcup_{j=1}^D L_j = V$). The $L_j$ are referred to as \emph{layers}, and $D$ is the \emph{depth} of the hypergraph.

    A weighted, layered hypergraph is called \emph{feed-forward} if for all directed hyperedges $\{c_{j_1}, \dots, c_{j_i}\} \rightarrow \{ c_{k_1} ,\dots c_{k_o} \} \in E$, there is an $\ell \in \{1,2,\dots, D-1\}$ such that $\{c_{j_1}, \dots, c_{j_i}\} \subset L_\ell$ and $\{ c_{k_1} ,\dots , c_{k_o} \} \subset L_{\ell + 1}$.
\end{definition}

Now, we integrate this notion of feed-forward, weighted, layered hypergraphs into our many-body physics-inspired Inductive Bias \ref{Bias:ManyBody}:

\begin{bias}{2FF}\label{Bias:Layered}
    For weighted, layered feed-forward many-body hypergraphs with layers $(L_1,\dots , L_D)$, we introduce the following modification of Inductive Bias~\ref{Bias:ManyBody}:

	The \emph{FeedForward} (FF) Inductive Bias is the same as Inductive Bias \ref{Bias:ManyBody}, except for the following restriction:
 
    The $h_{(i,o)}(\{ c_{j_1}, \dots, c_{j_i}\}, \{c_{k_1}, \dots , c_{k_o}\})$ have to satisfy the feed-forward condition, i.e. there is an $\ell \in \{1,2,\dots, D-1\}$ such that $\{c_{j_1}, \dots, c_{j_i}\} \subset L_\ell$ and $\{ c_{k_1} ,\dots , c_{k_o} \} \subset L_{\ell + 1}$.
	
	Furthermore, we require that all random walks couple out an action if all excitations are in layer $L_D$, or earlier.
\end{bias}

We are referring to this Inductive Bias with the label 2FF because in ~\ref{Appendix:OtherInductiveBiases} we will introduce some small modifications of this inductive bias called 2SF and 2DP. These modifications guarantee a polynomial upper bound on the random walk path lengths. However, in the main text, we will focus on Inductive Bias~\ref{Bias:Layered} because the worst-case paths are rarely encountered in practice.

Also for this inductive bias, we will see in ~\ref{Appendix:RelationHypergraphs} that it is compatible with the standard MEPS agent in Definition \ref{Definition:MEPSarchitecture} during inference with the standard probability assignment.

Now that we have formulated our Inductive Bias and its main variant for layered ECMs, we provide an example that discusses how these inductive biases get applied.

\begin{figure}[h!]
\centering
\includegraphics[width = 0.6\textwidth]{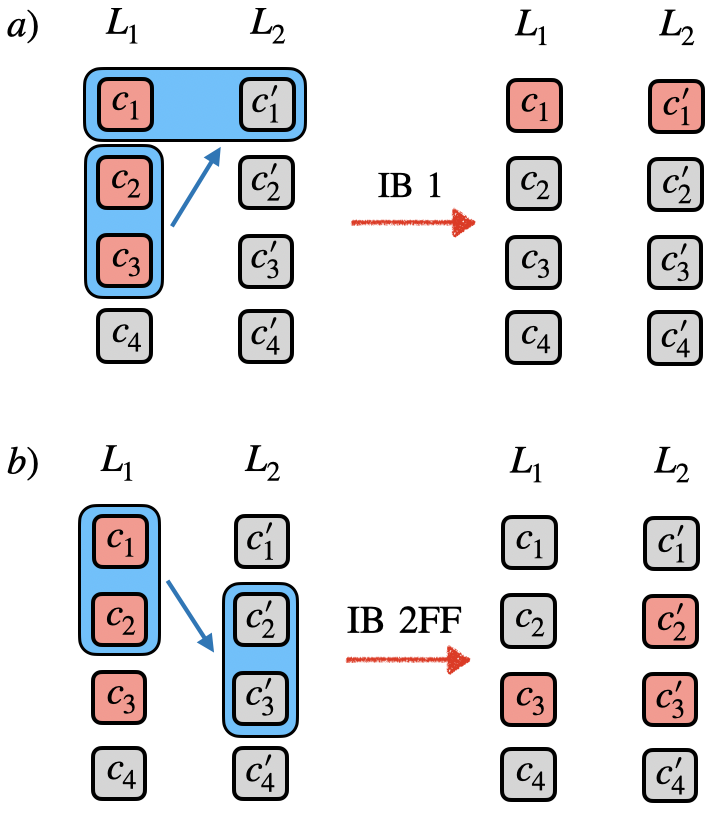}
\caption{An example illustrating random walk steps under different inductive biases, compare with Example \ref{Example:InductiveBiases}. Excited atomic clips are shown in red. The sampled hyperedge is shown in blue. Subfigure a) shows a deliberation step which is only allowed under Inductive Bias \ref{Bias:ManyBody}, because its codomain is in two layers. Also, it shows that an excitation moving into an occupied atomic clip gets discarded. Subfigure b) shows a typical transition under Inductive Bias \ref{Bias:Layered}. }

\label{Figure:InductiveBias2FF}    
\end{figure}

\begin{example} \label{Example:InductiveBiases}

Consider a simple 2-layer setting, with 4 atomic clips in each layer, see Figure \ref{Figure:InductiveBias2FF}: $V=L_1 \cup L_2$, with $L_1 = \{c_1 , c_2, c_3, c_4\}$ and $L_2 = \{c'_1, c'_2, c'_3, c'_4\}$. We only consider h-values with the same number of incoming and outgoing excitation numbers, and let no more than two excitations interact. That means $IO = \{(1,1), (2,2)\}$. Our current excitation configuration is $\{ c_1, c_2, c_3 \}$, meaning that we currently have an excitation in each of the atomic clips $c_1$, $c_2$, and $c_3$. 

With the weakest of the inductive biases, i.e. Inductive Bias \ref{Bias:ManyBody}, and choosing $E^{(i,o)} = E_{\mathrm{all}}^{(i,o)}$, our list $\mathcal H_{\mathrm{relevant}}$ of currently relevant h-values is:
\begin{enumerate}
    \item $h_{(2,2)}(\{c_m, c_n\}, \{c'_j, c'_k\})$ such that $j,k \in \{1,2,3,4\}$, $j < k$ and $m,n \in \{1,2,3\}$, $m < n$ 
    \item $h_{(2,2)}(\{c_m, c_n\}, \{c_j, c_k\})$ such that $j,k \in \{1,2,3,4\}$, $j < k$ and $m,n \in \{1,2,3\}$, $m < n$, and $\{j,k\} \ne \{m,n \}$
    \item $h_{(2,2)}(\{c_m, c_n\}, \{c_j, c'_k\})$ such that $j,k \in \{1,2,3,4\}$ and $m,n \in \{1,2,3\}$, $m < n$
    \item $h_{(1,1)}(c_m, c'_j)$ such that $j \in \{1,2,3,4\}$ and $m \in \{1,2,3\}$
    \item $h_{(1,1)}(c_m, c_j)$ such that $j \in \{1,2,3,4\}$, and $m \in \{1,2,3\}$, and $j \ne m$
\end{enumerate}
This list gets turned into probabilities, in our example by applying the softmax-function to the full list. Say, we sample $h_{(2,2)}(\{c_2, c_3\}, \{c_1, c'_1\})$ and apply it to our current configuration $\{c_1, c_2, c_3\}$. First, we remove the excitations in $c_2$ and $c_3$, giving us the configuration $\{c_1\}$. Next, we put excitations into $c_1$ and $c'_1$. However, $c_1$ already carries an excitation. We just keep this excitation as it is. So our next excitation configuration is $\{c_1, c'_1\}$. Note that our rule for dealing with already occupied atomic clips led to a reduction in the total number of excitations.

Our layered Inductive Bias \ref{Bias:Layered} differs from the previous situation in that the relevant many-body h-values are only items 1 and 4 from the numbered list above. Now, say that we sampled $h_{(2,2)}(\{c_1, c_2\}, \{c'_2, c'_3 \})$ and apply it to our current configuration $\{c_1, c_2, c_3\}$. First, we remove the excitations in $c_1$ and $c_2$, giving us the configuration $\{c_3\}$. Next, we insert excitations in $c'_2, c'_3$, giving us the full next excitation configuration $\{c'_2, c'_3, c_3\}$. We observe that while the feed-forward condition forces all excitations that move to move one layer forward, it allows excitations to stay behind in their old atomic clip in the old layer. Consider now an additional layer $L_3$. Inductive Bias \ref{Bias:Layered} allows us to continue with any transition $C_{\mathrm{in}} \rightarrow C_{\mathrm{out}}$ that has $C_{\mathrm{in}} \subset \{c'_2, c'_3\}$ or $C_{\mathrm{in}} = \{c_3\}$. 
\end{example}

Physically, this layered structure corresponds to situations encountered, for example, in integrated photonics chips performing quantum computation with several photons: The photons move forward in the lateral direction, but perform a quantum walk in the transversal direction~\cite{Flamini_Krumm2023, Photonics}. A common noise source of such chips is that photons get absorbed by the environment.

To quantify the resource advantages provided by our inductive biases, we first consider the costs associated with an unrestricted MEPS agent, which can be thought of as a standard PS agent by way of mapping the graph-based ECM to the hypergraph ECM defined in this work. For that purpose, we first make the following definition:

\begin{definition}
    A standard MEPS agent with ECM $(V, E, h)$ is called \emph{unrestricted} if all mathematically well-defined hyperedges are in $E$, i.e. if $E =  \big(\mathcal P(V)\setminus \{ \varnothing \} \big) \times \big(\mathcal P(V)\setminus \{ \varnothing \}\big)$.
\end{definition}

From this definition, one can quickly see that unrestricted MEPS agents have several costs associated to them that scale (at least) exponentially in the number of atomic clips $|V|$.

\begin{proposition}
    Consider an unrestricted MEPS agent with ECM $(V, E, h)$. Then:
    \begin{enumerate}
        \item[(a)] The number of trainable parameters is $(2^{|V|}-1)^2$. Therefore, the memory cost is also $\Omega (2^{2|V|})$.
        \item[(b)] At each deliberation/random-walk step, there are $2^{|V|}-1$ relevant h-values. In particular, at each deliberation/random-walk step, one must sample from a probability distribution with $2^{|V|}-1$ outcomes. 
    \end{enumerate}
\end{proposition}

\begin{proof}
    (a) follows from the statement $E = \big(\mathcal P(V)\setminus \{ \varnothing \} \big) \times \big(\mathcal P(V)\setminus \{ \varnothing \} \big)$, with $|\mathcal P(V)| = 2^{|V|}$ and $|A\times B| = |A| \times |B|$ for sets $A,B$. 

    (b) follows from the fact that for all $C_{\mathrm{in}} \in \mathcal P(V) \setminus \{ \varnothing \}$, each $C_{\mathrm{out}} \in \mathcal P(V) \setminus \{ \varnothing \}$ gives a relevant and separate h-value $h( C_{\mathrm{out}} | C_{\mathrm{in}})$. 
\end{proof}

These severe scaling costs make it very clear that inductive biases restricting the set of hyperedges or relevant h-values are crucial. 

We now analyze the costs of our Inductive Biases:

\begin{proposition} \label{Proposition:Bias1}
    Consider a MEPS agent obeying Inductive Bias \ref{Bias:ManyBody} or \ref{Bias:Layered}. Define $\max I := \max \{ i \ | \ \exists o: (i,o) \in IO \}$ and $\max O := \max \{ o \ | \ \exists i: (i,o) \in IO \}$, as well as $\max IO := \max \{ i + o \ | \ (i,o) \in IO \}$. Then
    the number of trainable parameters is $\mathcal O ( \max I \cdot \max O \cdot |V|^{\max IO})$.
\end{proposition}

\begin{proof}
    First, we note that $|IO| \le \max I \cdot \max O$. For each $(i,o)$, let us bound the number of $C_I, C_O \in \mathcal P(V)$ for the many-body h-values $h_{(i,o)}(C_I, C_O)$. Using binomial coefficients and Inductive Bias \ref{Bias:ManyBody}, this number is upper-bounded by $\begin{pmatrix} |V| \\ i \end{pmatrix} \cdot \begin{pmatrix} |V| \\ o \end{pmatrix} \le |V|^i |V|^o = |V|^{i+o} \le |V|^{\max IO}$. So, the number of many-body h-values for each $(i,o)$ is upper-bounded by $|V|^{\max IO}$. Since we have at most $\max I \cdot \max O$ choices for $(i,o)$, the total number of h-values is upper bounded by $\max I \cdot \max O \cdot |V|^{\max IO}$. Inductive Bias \ref{Bias:Layered} has even fewer many-body h-values than Inductive Bias \ref{Bias:ManyBody} alone would allow.
\end{proof}

\begin{remark}
    While Proposition \ref{Proposition:Bias1} bounds the number of many-body h-values, it leaves open the possibility that it is computationally expensive to determine which many-body h-values $h_{(i,o)} (C_{\mathrm{in}}, C_{\mathrm{out}})$ are relevant. However, that is not the case: given a configuration $\{c_{m_1}, \ldots , c_{m_x}\}$ (labelled such that $m_1 < \dots < m_x$) and any $(i,o) \in IO$, one just lists all the $\begin{pmatrix} x \\ i\end{pmatrix}$ subsets of $\{c_{m_1}, \dots , c_{m_x}\}$ that have cardinality $i$, and all $\begin{pmatrix} |V| \\ o \end{pmatrix}$ subsets of $V$ that have length $o$, and discards those not in $E^{(i,o)}$. This can be done by using an ansatz $C_{\mathrm{in}} = \{ c_{j_1}, \dots, c_{j_i}\}$ and using a for-loop that has $j_1, \dots j_i$ all run over $m_1, \dots, m_x$, with the extra condition $j_1 < \dots < j_i$. The number of for-loop iterations is clearly upper bounded by $x^i \le x^{\max I} \le |V|^{\max I}$. Similarly for the sets $C_{\mathrm{out}}$, we can use a for-loop with no more than $|V|^{\max O}$ iterations. We do not formulate this observation as a formal proposition because we do not wish to obfuscate the simple argument by getting too specific about the computational model used for resource counting.
\end{remark}

While our inductive biases reduce the number of trainable parameters and relevant transitions from exponential scaling to a polynomial scaling in $|V|$, the exponent of this polynomial scaling is determined by the interaction cutoff $\max IO$. One might wonder whether generically, $\max IO$ should be chosen as a function of $|V|$. Considering thought processes of humans in typical, everyday situations, it seems likely that there exist low values of $\max IO$ that should be successful on a large variety of problems (say, $\max IO \approx 10$). Humans are very successful at adapting to a large variety of domains. Despite this success, most humans can only combine a handful of facts simultaneously into one thought. 

While Inductive Bias~\ref{Bias:Layered} is naturally adapted to layered ECMs and only has an amount of trainable parameters polynomial in the number of atomic clips, it still allows for some pathological many-body hypergraphs for which there exist random walk paths that have a length exponential in the number of layers. Specifically, in Proposition \ref{Proposition:DeepToShallowIsTight} in ~\ref{Appendix:Complexity} we consider many-body hypergraphs with $o > i = 1$. These can start an ``avalanche'' of new excitations  every time one tries to move one excitation forward. If one removes these excitations in the worst order (from deep to shallow), we prove that these random walk paths have an exponential length.

While we do not expect these worst-case avalanche paths to matter in practice (an agent which indicates that almost all possible concepts matter is not very helpful), we still propose modifications of Inductive Bias~\ref{Bias:Layered} which explicitly make such paths impossible. 

Specifically, in ~\ref{Appendix:OtherInductiveBiases}, we introduce Inductive Bias 2SF and 2DP. Inductive Bias 2SF differs from 2FF by demanding that the shallowest excitations get removed first. In ~\ref{Appendix:Complexity}, we prove that Inductive Bias 2SF has a maximal random walk length of $\mathcal O( D \cdot \max_j{|L_j|} )$, which is linear in both the width and depth of the ECM. 

Such agents can be interpreted as forgetting the oldest facts/concepts/atomic clips first. If one wishes such old atomic clips (or rather, their semantics) to persist, one needs to model $E^{(i,o)}$ such that it allows one to copy-paste these atomic clips into deeper layers. 

Inductive Bias 2DP is the harshest one, it discards all passive excitations which did not contribute to a sampled transition. This is an agent which only keeps in mind atomic clips which are immediately relevant, but as a tradeoff its random walk lengths are upper-bounded by $\mathcal O(D)$, as we also prove in ~\ref{Appendix:Complexity}.

However, we consider exponential excitation avalanches to be pathological as they are not helpful for interpretability, and therefore restrict our attention in the main text to Inductive Bias~\ref{Bias:Layered}, which is the most flexible.

\section{Learning Scenarios} \label{Section:Numerics} 

In this section, we apply our methods numerically to three synthetic environments. The code can be found in our GitHub repository~\cite{GitHub}. The first environment is a small toy environment that allows us to understand the basic numerical properties of MEPS agents with different many-body inductive biases in a controlled setting. The second environment is an extension of the first with more actions and a mechanism for deception. Furthermore, its reward contains a contribution measuring the success of an attempted deception. The third environment is used to demonstrate chain-of-thought explanations in multi-layered MEPS agents operating in a real-world-inspired setting. It models a coarse-grained scenario for diagnosing broken computers. In each of the environments the unrestricted MEPS agent corresponds to a standard PS agent, which will serve as a baseline comparison to MEPS along with a standard 2-layer tabular Q-learning agent and a newly designed tabular PS-inspired multi-layer Q-learning agent for further benchmarking. Further details about the tabular PS-inspired multi-layer Q-learning agent will come in Subsection \ref{Subsection:ComputerMaintenance} . We choose the Q-learning algorithm because we believe it is the fairest and most straightforward comparison with our model in terms of explainability and performance, also noting that Q-learning is not the only method that MEPS is competitive with. 

\subsection{Invasion Game With Distraction} \label{Subsection:InvasionGameWithDistraction}

\begin{figure}[h]
    \centering
    \includegraphics[width = 0.7 \linewidth]{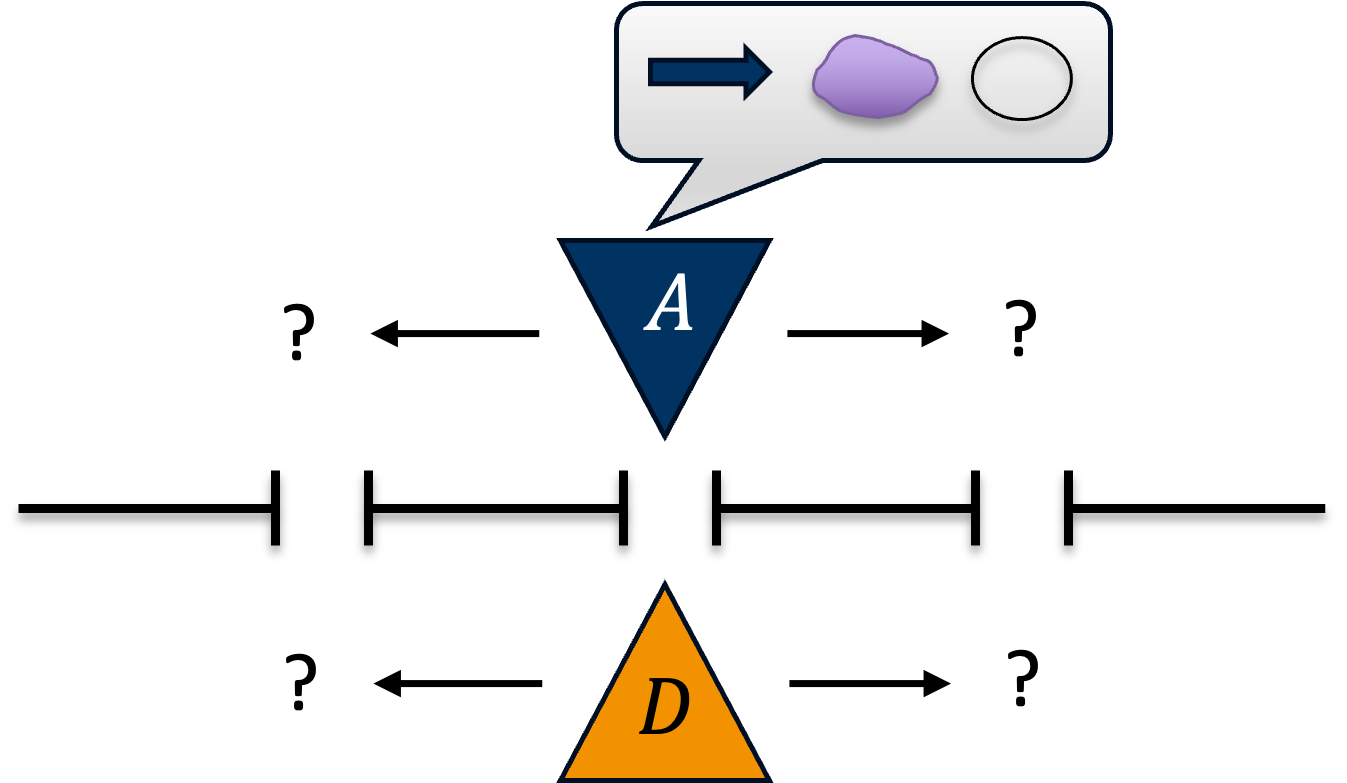}
    \caption{The defender $D$ must guess which door the attacker $A$ will go to based on a set of symbols shown to them, and block $A$. $D$ is rewarded for success and punished for failure.}
    \label{Figure:InvasionDiagram}
\end{figure}

The \emph{Invasion Game} is a standard toy environment~\cite{Briegel_Cuevas2012} to visualize the basic concepts of PS. Mathematically, it is a special case of so-called \emph{contextual bandits} problems~\cite{Bandits1, Bandits2}. The original environment considered a sequence of doors that an attacker would try to enter through, and which a defender (the agent) would attempt to block for a set number of rounds. During each round, the attacker would indicate to the defender via some abstract symbols which door they intend to visit, and the defender would receive a reward based on whether they guessed the correct door or not (as visualized in Figure \ref{Figure:InvasionDiagram}). Thus, the task of the defender is to infer the meaning of said symbols by learning the attacker's strategy. Here, we modify the Invasion Game such that it provides a simplistic environment showcasing the impact of different choices of few-body inductive biases. 

The different few-body inductive biases we will consider play a large role in determining the complexity of the agent, illustrating the general idea behind inductive biases: adapting the agent's biases and complexity to suit the task environment. For this purpose, our goal is to construct an environment that cannot be solved through consideration of only one excitation at a time, but can be solved perfectly by looking at two excitations and is ``overcomplicated'' when looking at three excitations simultaneously.

An environment achieving this goal is constructed as follows: Each round, the agent obtains a percept/observation of the form $(v_1, v_2, v_3) \in \{0,1,2\dots , 9\}^{\times 3}$. Here, each entry $v_j$ corresponds to a value of an observable $j$. For each observable $j$ and each value $v_j$, we associate an atomic clip denoted ``$\mathrm{obs_j : v_j}$'' in the percept layer of our two-layer agents. Therefore, each percept is coupled in by putting three excitations into the corresponding atomic clips in the percept layer. As an action, the agent has to pick exactly one of two doors. These two choices are represented by atomic clips ``$a = 0$'' and ``$a = 1$'' in the action layer of the agent. The action gets coupled out as soon as there is an excitation on one of the atomic action clips. 

For the hyperedges, we built in the domain knowledge that allowed actions pick exactly one door. Because of our decision to directly couple out actions as soon as there is an excitation in the action layer, we restrict the allowed hyperedges to those that have their tail in the percept layer and their head in the action layer. We do not allow transitions within the percept layer, because these would have the interpretation that the observables had suddenly changed.  With these choices, a standard MEPS agent without further restrictions is equivalent to the $(3,1)$-agent from the few-body inductive bias agents that we specify now: 

For $i \in \{1,2,3\}$, we consider two-layered (percept+action layer) agents with many-body inductive bias using $IO = \{(i, 1)\}$, and use percept and action layers as described above. This allows us to directly focus on the difference caused by different choices of $i$. 

To keep the comparison of the cases for different $i$ as clean and simple as possible, we sample the percepts $(v_1,v_2,v_3)$ uniformly i.i.d., meaning that we do not need the forgetting and glow mechanisms in the update rule (this corresponds to $\gamma = 0$ and $\eta = 1$, respectively). This reduces the update rule to $h^{(n+1)}_{(i,o)} = h^{(n)}_{(i,o)} + r$, with $r$ the reward. For each percept $(v_1,v_2,v_3)$, there is exactly one right action $a$. This action depends non-trivially on both of the first two observables. We pick the right action to be $a = v_1 + v_2 \ \mathrm{mod} \ 2$. The value of the third observable, $v_3$, is just a useless distraction. 

The $(1,1)$-agent has many-body h-values of the form $h^{(n)}_{(1,1)}(\{obs_j : v_j\}, a)$. Since it can consider only one observable per decision-making process, it cannot learn to deterministically map the values of the first two observables to the right action. There are $2\cdot 3\cdot 10 = 60$ many-body $h$-values (i.e. trainable parameters) for this agent.

The many-body $h$-values of the $(2,1)$-agent are of the form $h^{(n)}_{(2,1)}(\{obs_j : v_j, obs_k : v_k\}, a)$ for all $j < k$. There are $2 \cdot \begin{pmatrix} 3 \\ 2 \end{pmatrix} \cdot 10 \cdot 10 = 600$ many-body $h$-values/trainable parameters for this agent. This agent has exactly the right inductive bias because its many-body $h$-values $h^{(n)}_{(2,1)}(\{obs_1 : v_1, obs_2 : v_2\}, a)$ exactly encode the information needed for the right action.

The $(3,1)$-agent, which represents a standard PS agent here, has many-body $h$-values of the form $h^{(n)}_{(3,1)}(\{obs_1: v_1, obs_2: v_2, obs_3: v_3\}, a)$. These are $2\times 10^3 = 2000$ trainable parameters, significantly more than for the $(2,1)$-agent. Its many-body $h$-values distinguish between different values of the distraction $v_3$, so it is reasonable to expect that this agent also trains slower. 

All MEPS agents use the softmax function with $\beta = 1.0$ to convert $h$-values to probabilities, and all $h$-values are initialized to $1.0$. The standard tabular 2-layer Q-learning agent has a discount factor of  $\lambda = 0$ and a learning rate of $\alpha = 1$ to match the analogous settings used in the MEPS agents. 

After each action, the agent obtains a reward of $+1$ for a right answer and a harsh negative reward of $-10$ for a wrong answer. For the $(2,1)$- and $(3,1)$- agents, this practically prevents the transition from being sampled again. This allows us to map the advantage of the $(2,1)$-agent concerning the number of trainable parameters to an advantage in training time over the $(3,1)$-agent. This mechanism does not apply to the $(1,1)$-agent, since it has no transitions that can deterministically choose the right action.

We train over 10000 rounds, each consisting of one percept-action pair, and average rewards over 100 consecutive rounds. Furthermore, we average the learning curves over 50 agents using the same inductive bias but different random number generator seeds. The results are shown in Figure \ref{Figure:InvasionDistraction}, and confirm our expectations: The 1-body agent cannot solve the problem, the 2-body agent learns to solve the problem perfectly and learns the fastest. The 3-body agent also learns to solve the problem, but it learns more slowly. From the standard deviations (shaded areas), we see that the fluctuations are negligible. We can also see that the standard tabular 2-layer Q-learning agent performs slightly better than the standard PS agent (3-body MEPS), but worse than the 2-body MEPS with the optimal inductive bias.  These results confirm that the capacity of the (2,1)-agent to discard the distraction allows it to learn faster than the standard PS and Q-learning agents, which both assign separate h/Q-values for each percept.

\begin{figure}[h!]
    \centering
    \includegraphics[width = 0.7 \linewidth]{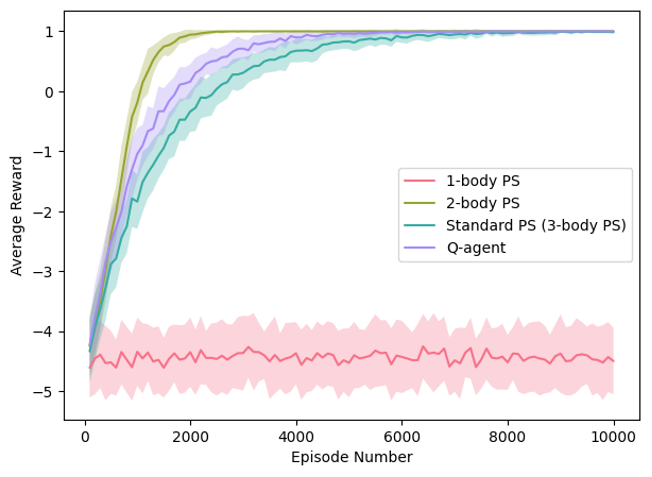}
    \caption{The average reward learning curves in the Invasion Game with Distraction for MEPS agents with the inductive biases discussed in section \ref{Subsection:InvasionGameWithDistraction} and a standard tabular 2-layer Q-learning agent. Note that, in this environment, the 3-body PS is equivalent to standard PS and is used as the comparison of MEPS to standard PS. Each curve is averaged over an ensemble of 50 agents.}
    \label{Figure:InvasionDistraction}
\end{figure}

\subsection{Deceptive Invasion Game} \label{Subsection:DeceptiveInvasionGame}

We now consider an extension of the previous Invasion Game With Distraction environment, where the defender now has access to a greater number of possible actions they can take and the deterministically correct answer to which door the attacker will visit depends on the parity of the sum of the first two observables: an even sum means that the symbol shown to the defender that represents the door number actually corresponds to the door the attacker will go to, while an odd sum means that the attacker will go to the next door over. The third observable maintains its original meaning and purpose from the Invasion Game with Distraction. We call this environment the \emph{Deceptive Invasion Game}. The goal of the defender in this environment differs from the previous environment in that they must learn to also associate the parity of the sums of the first two observables with the truth value of the first observable, and to correctly predict that an odd parity for this sum entails the next door over being the actual door the attacker will go to. 

The first and third observables are comprised of the same values from the previous environment, while the second observable can take values in the range 10-13 and the range of values for the defender's actions now mirrors that of the first observable. This has the effect of inflating the number of trainable parameters used by the agent for each of the many-body cases considered previously. The $(1, 1)$-agent now has $(10+10+4)\cdot 10 = 240$ many-body h-values, while the $(2, 1)$-agent has $ 10^3 + 10\cdot 4 \cdot 10 + 4 \cdot 10 \cdot 10  = 1800$, and the $(3, 1)$-agent has $4\times 10^3 = 4000$. Note that the $(3, 1)$-agent is again representing a standard PS agent here.

The first observable is interpreted as the door announced by the attacker. The reward is determined based on even and odd parity cases of the first two observables. A reward of +2 is given to the defender if they choose the door shown to them by the attacker in the even parity case, while in the odd parity case, a reward of +2 is given if the agent chooses the next door over. If the defender chooses any other door then they receive a harsh negative reward of -10 to effectively deter them from selecting that option during future deliberations. In the odd parity case, if the agent picked the door announced by the attacker, they receive an additional penalty of -1 (i.e. -11 in total), interpreted as the defender being deceived by the attacker. The agent also gets an additional punishment of $-1$ for picking the wrong door in the even case.

What is meant by deception in this environment is that the attacker can do something different than what they convey in the percepts shown to the defender. From an interpretability perspective, this is what an external observer would hypothesize is happening if they observed the attacker's movements and had access to the reward structure of the environment. A query to the defender after training would also reflect this if learning was successful. However, from the defender's point of view, since they do not know the meaning of any of the attacker's symbols \textit{a priori}, they will blindly learn the policy that maximizes the reward received from the environment and have no concept of "deception" (unless they are somehow given this concept). 

It is again expected that the $(2, 1)$-agent will reach the optimal policy the quickest, followed by the $(3, 1)$-agent taking more time due to the processing of irrelevant percepts represented by the third observable. Due to the more complex reward structure of this environment compared with the Invasion Game With Distraction, the $(1, 1)$-agent's task will become even less feasible than it already was because it has to cut through the attacker's deception on top of the already present obstacles in the Invasion Game With Distraction environment.

Looking at Figure \ref{Figure:InvasionDeception}, we can see similar behaviour to the agent in the Invasion Game with Distraction: the 2-body agent reaches the optimal policy the quickest, while the 3-body agent still learns the optimal policy more slowly than the 2-body agent. The 1-body agent unsurprisingly still cannot learn the optimal policy, but notice that it gets stuck near one of the worst policies. This is a puzzling observation since the 1-body agent should be able to represent much better policies than it actually learns: an agent that always looks at the door announced by the attacker and always picks that door (or always picks the next door) should achieve a significantly better average reward than random guesses. 

\begin{figure}[h!]
    \centering
    \includegraphics[width = 0.7 \linewidth]{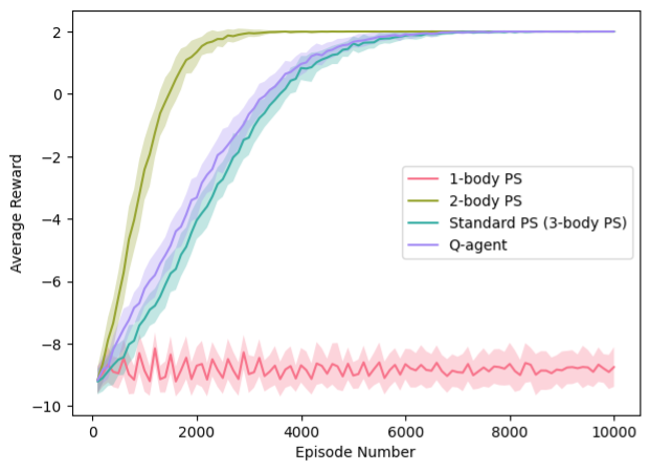}
    \caption{The average reward learning curves in the Deceptive Invasion Game for MEPS agents with the inductive biases discussed in section \ref{Subsection:DeceptiveInvasionGame} and a standard tabular 2-layer Q-learning agent. Similar as in Figure \ref{Figure:InvasionDistraction} the 3-body PS is equivalent to standard PS and is used as a baseline comparison to MEPS. Each curve is averaged over an ensemble of 50 agents.}
    \label{Figure:InvasionDeception}
\end{figure}

We believe that the explanation for this puzzle is the following: while the described policies are much better than random guesses, the involved transitions still receive negative rewards on average. The update rule decreases the corresponding $h$-values, making the transitions less likely. Meanwhile, the $h$-values of transitions that are never picked are never decreased. In other words, the less severe average punishments are compensated by applying them more often. This argument implies that a 1-body agent initialized with the better policies would unlearn these policies since the involved transitions get punished often, even if individual punishments are less severe on average. 

The performance of the standard tabular 2-layer Q-learning agent here is almost identical to that in Subsection \ref{Subsection:InvasionGameWithDistraction}, but now the performance gap between it and the standard PS agent (3-body MEPS) is even smaller. This also means that the 2-body MEPS agent has a much greater advantage over the standard tabular 2-layer Q-learning agent than before, which may be attributed to the increased memory size of the agents in this environment and the increased complexity of the reward function in the environment.

\subsection{Computer Maintenance} \label{Subsection:ComputerMaintenance}

The final environment we consider is that of diagnosing and fixing a broken computer, which we call the \emph{Computer Maintenance} environment. Computer repair is an everyday task that can be quite complex, with many possible causes in various systems giving rise to any particular problem; and yet, the task is accomplished daily, making it not so complicated as to be completely intractable. Computer technicians can manage this task complexity because they can keep track of multiple variables simultaneously, which is a daunting and ultimately unfeasible situation for an agent restricted to single excitations only. It is for these reasons that we choose the Computer Maintenance environment to highlight the capabilities of an agent in a complex environment who considers multiple excitations throughout a non-trivial chain-of-thought, as well as the usefulness of the inductive biases in reducing ECM size. 

We choose to visualize the Computer Maintenance environment as follows. A customer enters a computer repair shop with a broken computer and asks the technician to find the root causes of the issues and fix them. While the technician is diagnosing the problem, several symptoms indicative of possibly many underlying causes of the problem become apparent, which could include a combination of software and hardware issues. The technician must now identify the relevant components and assert a hypothesis about what the causes of the symptoms are given these components, then apply appropriate fixes that will hopefully solve the problem. Along with an explanation of the underlying causes of the problem, the technician also presents an invoice to the customer that states the amount of time and resources used to repair the computer.

\begin{figure*}[t!]
    \centering
    \includegraphics[width = 1.0\textwidth, height = 6.5cm]{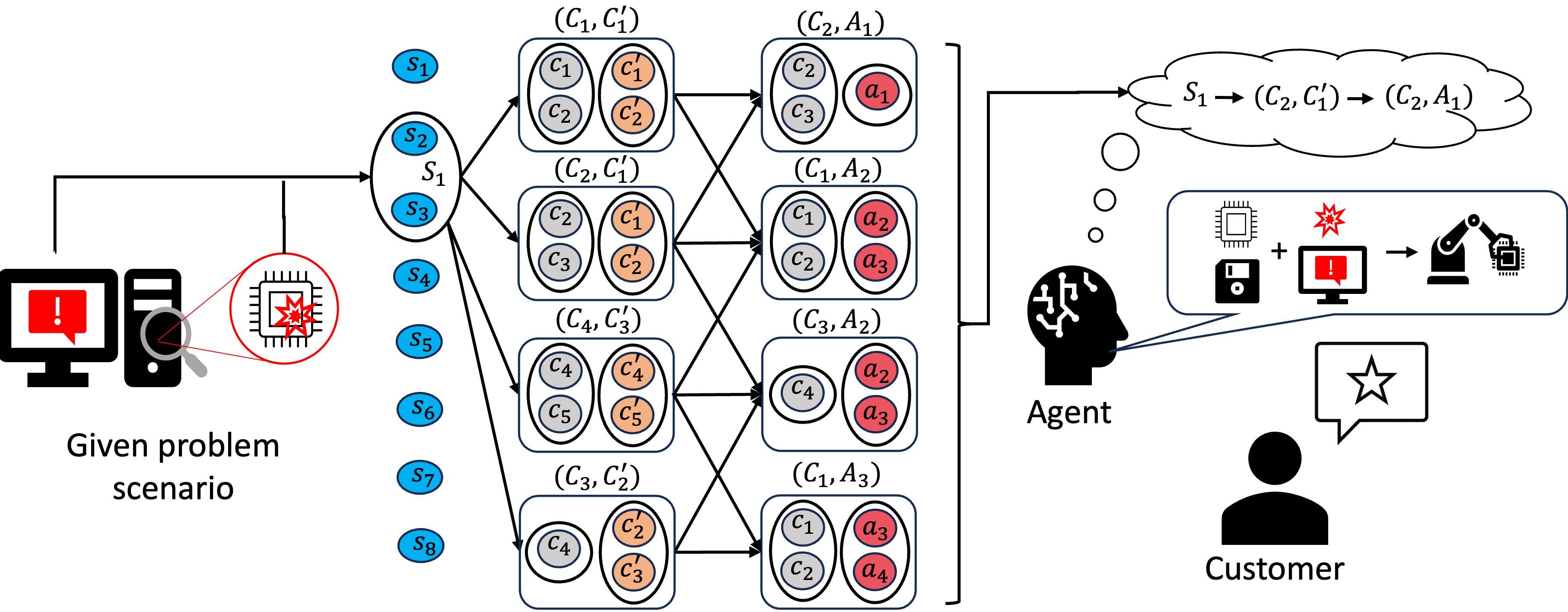}
    \caption{A schematic illustration of one environment step and the MEPS agent's ECM during this step (in the Computer Maintenance environment). The symptoms are represented by the blue dots; the components by grey dots; the causes by orange dots; and the fixes by red dots (the numbers correspond to the place of that object in Table \ref{Table:CompMaintenanceClips}). The boxes denote pairs of excitation configurations that the agent can transition to/from. The agent first observes a pair of symptoms that get coupled into a percept source hyperedge $S_1$ which triggers deliberation through pairs of components/causes ($C_i, C^{'}_j$) and components/fixes ($C_i, A_j$), ending with a chain-of-thought that the agent then converts into an explanation to the customer who determines their reward. As an example, consider a clumsy owner who drops their computer on the ground, causing physical damage to the motherboard (MoBo) that results in a software error in the storage unit (SSD), leading to physical damage in the SSD, which then also causes physical damage to the MoBo in a feedback loop. Such a situation in which software interacts with hardware can occur when faulty programs overuse resources (resulting in heat or electrical damage to the computer) or create conflicting processes that lead to this. In the Computer Maintenance environment, this scenario would be coded as: [[`files disappearing', `visible markings on components'], [`MoBo', `SSD'], [`physical damage', `software damage'], [`replace components']].}
    \label{Figure:CompMaintenance}
\end{figure*}

Translating the previous description to a reinforcement learning setting, the computer technician is interpreted as our MEPS agent who can receive sets of symptoms generated from the environment as percepts and perform actions on the environment in the form of selecting sets of components and corresponding fixes to those components (as a pair of subsets). The environment contains a set of lists of the possible symptoms, components, causes, and fixes whose text descriptions are encoded as integers such that the agent is unaware of the association between the two \textit{a priori}; Table \ref{Table:CompMaintenanceClips} displays them.

\begin{table}
\centering
\begin{tabular}{| p{.19\textwidth} | p{.11\textwidth} | p{.17\textwidth} | p{.19\textwidth}|}
\hline
Symptoms & Components & Causes & Fixes \\
\hline
PC overheating: 1 & CPU: 15& physical damage: 20& replace components: 11\\
\hline
files disappearing: 2 & SSD: 16& software damage: 21& install missing software: 12\\
\hline
visible markings on components: 3 & MoBo: 17& malware: 22& cooldown computer: 13\\
\hline
unexpected shutdowns: 4 & PSU: 18& faulty: 23& run antivirus: 14\\
\hline
slow performance: 5 & OS: 19& not connected: 24&\\
\hline
old hardware: 6 & & &\\
\hline
strange noises: 7 & & &\\
\hline
software glitches: 8 & & &\\
\hline
 blue screen: 9& & &\\\hline
 no internet: 10& & &\\\hline
\end{tabular}
\caption{The symptoms, components, causes, and fixes in the Computer Maintenance environment along with their integer encodings.}\label{Table:CompMaintenanceClips}
\end{table}

Elements from each of these sets are then combined to form what we call `scenarios,' which are used to fix and specify a specific problem with a unique goal state that defines the length of a training episode: the agent repeatedly applies their policy until reaching the goal state, marking the end of the episode. At the beginning of each episode, a scenario is sampled uniformly at random and the corresponding subset of symptoms contained in it are then used for that episode. The chosen percept is fixed for the duration of the episode to more closely emulate the real situation assuming no new problems arise during the repair process and that the symptom set distinguishes one problem from another. Allowing new symptoms to arise during each step would confuse the agent on what the actual problem was since a given symptom set typically only corresponds to a small group of issues, thus preventing any explanation that agrees with the specified scenario.

The agent must learn to navigate and solve 44 different scenarios in this work, which we constrained to have at most 3 elements per category to better demonstrate the savings accrued from reducing the size of the ECM. An example scenario is detailed in Figure \ref{Figure:CompMaintenance}. It is immediately clear from this example scenario that the feedback loop between the MoBo and the SSD demands that these components and the associated causes all be treated together if the agent hopes to solve the problem -- something the MEPS agent is better suited to handle. The only hope that the single-excitation agent has of solving the scenario is to exponentially inflate the size of their ECM so that all possible combinations of the environment variables are distributed in all atomic clips, which is essentially the strategy of the unrestricted MEPS agent. If each environment variable can take sufficiently many values, the previous strategy will undoubtedly fail since the random walk path in the agent's ECM that represents the correct solution to a given scenario will have a vanishingly small probability of occurring.

To implement longer chains-of-thought and highlight the benefit of moving from the single- to multi-excitation agent case, we incorporate a hidden layer into our MEPS agent where the agent's hypotheses about the underlying causes of the problem are represented by pairs of sets of components and causes. Therefore, we use two many-body $h$ values: the first for transitions between the percept layer and hidden layer, and the second for transitions between the hidden layer and action layer. This change to a 3-layer agent and the use of pairs of subsets as hidden and action layer elements introduces some modifications to how the inductive bias is applied compared to the previous two Invasion Games. Firstly, because we have pairs of subsets for some layers, a many-body cutoff applied to the pair will sometimes force the agent to consider a subspace for one of the subsets that is larger than the scenario demands, since the scenario elements can be subsets of unequal size in general; this slows down learning unnecessarily, so many-body cutoff values are now assigned for each element in the pairs. It is also necessary to have $IO$ contain all values up to and including a specified many-body cutoff value, for the same reason stated in the previous sentence. Lastly, because we have two many-body $h$ values now, we can specify different sets of many-body cutoffs for each of them. 

An inductive bias agent will be represented with the following notation: $\{n_s, [n_{hc}, n_c], [n_{ac}, n_f]\}$, which we call an inductive bias configuration. Here, $n_s$ is the many-body cutoff on the number of symptoms in the inductive bias agent's percept layer; $n_{hc}$, the cutoff on the number of components in the hidden layer; $n_{c}$, the cutoff on the number of causes in the hidden layer; $n_{ac}$, the cutoff on the number of components in the action layer; and $n_{f}$, the cutoff on the number of fixes in the action layer. Since we are discarding leftover excitations, and because we have a layered, feed-forward architecture, the inductive bias used here corresponds to Inductive Bias 2DP from ~\ref{Appendix:OtherInductiveBiases}. It differs from Inductive Bias \ref{Bias:Layered} only in the fact that it discards passive excitations, i.e. $C_{\mathrm{out}} := C^{(o)}_{\mathrm{out}} $. We can calculate the number of learning parameters $N_l$ for each agent from Equation \eqref{Equation:num_learn_param} in \ref{Appendix:BrokenCompTrainDetails}, which we will use throughout the rest of this section. 

We do not use the glow or forgetting mechanisms in the learning process for our MEPS agent due to the following factors. 1) The distribution that governs symptom and scenario generation within the environment does not change with time, it is simply always the same uniform distribution; and 2) the fact that the actions do not directly influence the percepts within an episode since the percepts are frozen at the beginning of the episode.

We choose the inductive bias agent with configuration $\{2, [3, 2], [3, 2]\}$ to train on the Computer Maintenance environment. This configuration represents the agent with the smallest ECM, at $N_l = 114375$ trainable parameters, necessary to solve all of the scenarios considered in the training set; it is expected that this agent will be able to reach near-optimality in the fewest number of steps. We also use the unrestricted agent (serving as the standard PS agent) and a tabular PS-inspired multi-layer Q-learning agent (details to follow), both with $N_l = 1429968$ trainable parameters, for comparison against the inductive bias agent; a difference of about 13 times the number of trainable parameters. Since the standard PS (unrestricted) and tabular PS-inspired multi-layer Q-learning agents see the whole space, they are coded differently to take advantage of the extra speed that certain array structures are endowed with. They always choose the full percept/intermediate clip/action configuration at the start of each deliberative phase and also use a layered, feed-forward ECM. We can represent the standard PS (unrestricted) agents' architecture using the configuration $\{N_s, [N_c, N_{ca}], [N_c, N_f]\}$ (defined in \ref{Appendix:BrokenCompTrainDetails}). The standard PS (unrestricted) and tabular PS-inspired multi-layer Q-learning agents are expected to reach the optimal policy but using many more steps than the inductive bias agent as they are forced to consider a multitude of irrelevant transitions. Note that the inductive bias architecture requires much more real elapsed time than for the standard PS (unrestricted) or tabular PS-inspired multi-layer Q-learning agents (since the code requires better python control flow), so any comparisons between them will be made on the level of total steps taken on average to reach the maximum reward.

Figure \ref{Figure:CompMaintenance} illustrates the process leading up to the agent receiving a reward from the environment. The reward is decomposed into two separate mechanisms that are each applied to different $h$-values, which we call the hypothesis and plausibility rewards. We do this to avoid washing out the percept information that tends to happen when blindly applying the standard PS update rule to intermediate layers. The splitting of the reward motivates a modification of the standard tabular 2-layer Q-learning agent to the multi-layer setting so that it can be compatible with the computer maintenance environment, and thus a fair comparison can be drawn between it and our model. The architecture of PS provides a good general template for how to accomplish this modification, which is actually suitable for any tabular method provided one takes appropriate measures to address the credit assignment problem. The idea is to define a tabular memory corresponding to each layer and take the existing tabular deliberation rule and apply it to each layer, taking the output of each layer and feeding it into the next one as the input. The tabular update rule is then applied to the tabular memory corresponding to each layer. For the Q-learning algorithm, we define the set of Q-tables at time $t$  as $\{Q^{(t)}(C_0, C_1), Q^{(t)}(C_1, C_2), \ldots, Q^{(t)}(C_{n-1}, C_n)\}$, where $C_0$ are the percepts, $C_n$ are the actions, and $C_i$ are the hidden memory elements corresponding to layer $i$. The deliberation and update rule for an arbitrary layer $i$ are then
\begin{align}
C^{*}_{i+1} &= \text{arg}\max_{C_{i+1}}{Q^{(t)}(C^{*}_i, C_{i+1})}\\[1ex]
Q^{(t+1)}(C_i, C_{i+1}) &= (1-\alpha_i)Q^{(t)}(C_i, C_{i+1})+\alpha_i \left(R^{(t)}+\lambda_i\max_{C'_{i+1}}{Q^{(t)}(C'_i, C'_{i+1})}\right)\,,
\end{align}
respectively. Here $\alpha_i$ is the learning rate, $\lambda_i$ is the discount factor, $C^{*}_i$  is a previously selected memory element, and $C'_i$ is a memory element corresponding to the deliberation chain that originated from the newly generated percept received after the agent's most recent action, each for layer $i$. Other tabular methods can be similarly defined but we will focus on the one defined above for the remainder of this section. Note that while the method described above can promote the architecture of tabular methods to be functionally similar to that of MEPS, this does not endow the agent with the same semantics as MEPS. The PS-inspired multi-layer Q-learning agent described above, for example, cannot have its deliberation be thought of as a random walk through a network of clips - the reason being that Q-learning picks the highest Q-value to choose its action. Therefore, MEPS should be thought of as conceptually distinct, with the ability to inspire new methods in other parts of machine learning.

The hypothesis reward is measured based on how close each intermediate layer element is to its corresponding partner in the given scenario; it is applied to the $h$-values between the percept and intermediate layers. If those elements match exactly, a reward of +5 is given, otherwise, a large penalty of -10 is applied. An additional component of the reward penalizes the agent by up to -1 if they pick more elements than there are in the given scenario; an attempt to discourage the agent from selecting the maximum number of elements permitted by their inductive bias (becomes more important for increasing many-body cutoff size). 

The plausibility reward, which is applied to the $h$-values between the intermediate layer and the action layer, has three different components to it. The first component has the same structure as the hypothesis reward but the values are quadrupled for the penalty on too many elements and it is for the agent's chosen elements from the action layer instead of the intermediate layer. The second component checks whether the chosen components from the intermediate layer match those chosen from the action layer; a reward of +1 is given if they match perfectly, +0.25 if all of the action layer elements match but more intermediate layer elements are chosen then action layer elements, and -2 otherwise. This component of the reward can be interpreted as an internal consistency mechanism that encourages the agent to form coherent explanations between hypothesis and action. Note that it does not directly refer to the scenario key and is thus a general mechanism that can aid learning on scenarios not considered in the original training set. The third component of the reward checks whether the agent's chosen fixes make sense with respect to the underlying causes they identified in their explanation. This is judged based on certain causal relationships between what the causes and fixes each refer to, that are put in by hand. For example, implementing the fix `replace components' would be justified if the cause of the problem was suspected to be `physical damage' or `faulty,' since both indicate problems with the hardware that are only fixable by physically removing them. However, if the suspected cause was simply `malware,' then only selecting `replace components' would not be justified because this is a fix one implements when there are hardware problems, not software. Now, for each possible fix there are specific causes that need to be selected for the agent to receive the associated reward of +0.3 divided by the number of fixes specified in the chosen scenario key $n_{\mathrm{fix}}$; this ensures that the agent can converge to the optimal reward since some scenarios require different numbers of fixes to solve than others. A penalty of up to $-4/n_{\mathrm{fix}}$ is given if the proper causes are not selected. Aside from the use of $n_{\mathrm{fix}}$, the third component of the plausibility reward is also independent of the scenario key, adding another general mechanism for augmenting learning beyond the training set.

Finally, if all elements in the agent's explanation exactly match those in the scenario key, the agent receives a big bonus to both rewards of +15 to amplify the probability of selecting this deliberation path again in future episodes. Once this bonus is triggered it signals an end to the episode and a fresh set of symptoms corresponding to another scenario are then selected. Before receiving this signal and before the agent factors in the reward into their update rule definition \ref{Definition:MEPSlearning}, the reward is first sent to an external reward shaping function defined by Equation \eqref{Equation:RewardShape} whose purpose is to discourage the agent from taking too many steps within an episode, intending to curtail arbitrarily long episodes; details about this function can be found in \ref{Appendix:BrokenCompTrainDetails}.

All MEPS agents for each of their layers use the softmax function with $\beta = 1/2$ to convert $h$-values to probabilities and all $h$-values are initialized to 1. The tabular PS-inspired multi-layer Q-learning agent uses the same values for their hyperparameters from the previous Invasion Game environments for each layer. The results of the training are shown in Figure \ref{Figure:BrokenCompInductiveBias}. We can see that the inductive bias agent achieves near-optimal values for both the hypothesis and plausibility rewards using far fewer steps but requiring more episodes than both the standard PS (unrestricted) and tabular PS-inspired multi-layer Q-learning agents do, as seen in the total step number per episode curve, which better illustrates the savings from the inductive bias as the two average reward plots do not show how many total steps have elapsed during training. The total step number per episode curve also shows roughly exponential decay, which is indicative of the agent first trying out scenario configurations randomly and then always picking the right configuration immediately. To quantify the difference in step number to the optimal reward, we calculated the total average steps taken over 600 episodes for all agents: the inductive bias agent takes roughly 16766 steps, while the standard PS (unrestricted) and tabular PS-inspired multi-layer Q-learning agents each take roughly 35817 and 115261 steps, respectively. Together with the complexity-theoretic results from \ref{Appendix:Complexity}, we only expect this difference in step number to grow. Notice also how the tabular PS-inspired multi-layer Q-learning agent requires 79444 more steps on average than the standard PS (unrestricted) agent does, and that it doesn't follow a strict exponential decay curve for the average total step number. We suspect the reason for this is that, because the tabular PS-inspired multi-layer Q-learning agent has a deterministic policy with no exploration, it finds the correct scenario relatively quickly but then has difficulty adapting to new scenarios, hence the initial increase in average step number. Although the standard PS (unrestricted) and tabular PS-inspired multi-layer Q-learning agents do converge to the optimal rewards in fewer episodes, it clearly takes more steps to do so because they both have roughly 13 times the number of trainable parameters as the inductive bias agent has. The fact that we get comparable performance from the inductive bias agent using far fewer parameters justifies the relatively small fluctuations around the optimal value for each of the reward curves, which are expected to decrease as more agents are included in the average.

\begin{figure}[t!]
    \centering
    \includegraphics[width = 0.55 \linewidth]{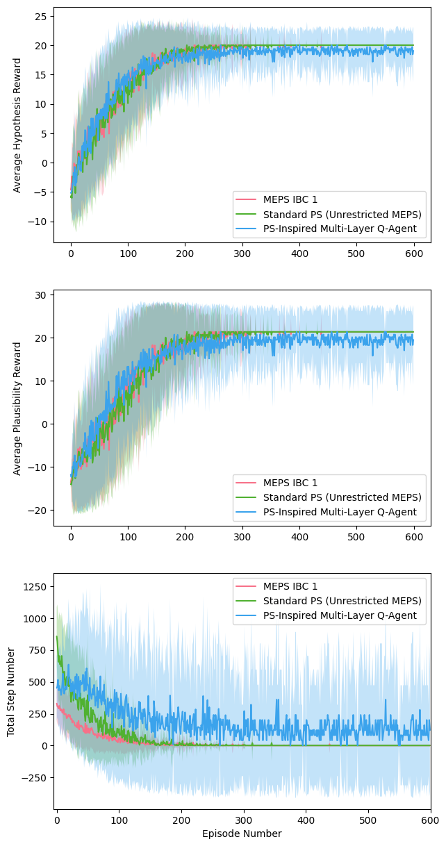}
    \caption{The average hypothesis and plausibility rewards, and step number per episode learning curves in the Computer Maintenance environment for the agent with inductive bias configuration $\{2, [3, 2], [3, 2]\}$ (IBC 1), the standard PS (unrestricted) agent, and the tabular PS-inspired multi-layer Q-agent. Each agent is trained for 1000 episodes (600 shown), where the number of steps taken within an episode varies from 1 (optimal) to around 800 (worst). Note that the x-axis of the total step number plot is restricted to 600 episodes so that the details of the early episodes can be better resolved, since the curves beyond this point are essentially flat. Each curve is further averaged over an ensemble of 50 agents. The standard deviation around the curves (shaded areas) is used to represent fluctuations but should not be interpreted as points that the individual agents have necessarily visited.}
    \label{Figure:BrokenCompInductiveBias}
\end{figure}

\section{A Quantum Motivation of the Inductive Bias and a Path Towards Quantization}
\label{Section:Quantum}

In Section~\ref{Section:Bias}, we gave a qualitative, conceptual motivation of our inductive biases. There, we identified that the size of the power set $\mathcal P(V)$ can lead to exponential costs, because in general all of its elements might be valid excitation configurations for which we have to define $h$-values. 

To get around this issue, we took a look at many-body physics, in which fundamental scattering processes typically only involve a handful of particles. This guided us to formulate a classical inductive bias in which complex transitions from observations to actions are given by successions of few-body transitions. 

However, the analogy to quantum many-body physics can be worked out more formally, and this in particular suggests a natural quantization for analog quantum computers. Therefore, in Section~\ref{Section:Motivation} we will work out this formal analogy, while in Section~\ref{Section:QuantumMEPS} we will provide some basic considerations for an actual quantum implementation.

\subsection{Full Quantum Motivation}
\label{Section:Motivation}

In this section, we will provide more formal details on the analogy to quantum many-body physics which inspired our inductive biases. For that purpose, we start by explaining how excitation configurations are represented in quantum physics. Specifically, each atomic clip in our hypergraph can be interpreted as a mode in quantum many-body physics and in the formalism called \emph{second quantization}~\cite{ManyBody1,ManyBody2}, each excitation configuration is represented by a vector $\ket{n_1, \dots, n_{|V|}}$ in a Hilbert space, where $n_j$ is the number of excitations in atomic clip/mode $j$.

Time evolution over the time duration $\Delta t$ is described by an operator $U(\Delta t) = e^{-i\Delta t \cdot H}$, where $i$ is the complex unit, $H$ is an operator called the \emph{Hamiltonian}, and $e^{\bullet}$ is the matrix exponential. This means after a duration $\Delta t$, a many-body system starting in a state $\ket{n_1, \dots, n_{|V|}}$ is afterwards described by the state $e^{i \Delta t \cdot H}\ket{n_1, \dots, n_{|V|}}$.

Important for our inductive bias are the typical expressions for $H$. For this, we require the \emph{ladder operators} $a_j$ and $a^\dagger_j$, with $^\dagger$ denoting the hermitian adjoint (i.e. $A^\dagger = \overline{A^T}$, with $\overline{\bullet}$ denoting complex conjugation) and $j \in V$. $a^\dagger_j$ adds one excitation to atomic clip $j$, i.e. $a^\dagger_j \ket{\ldots, n_j, \ldots} \propto \ket{\ldots, n_j+1, \ldots}$ and is called a \emph{creation operator}. Similarly, $a_j$ removes an excitation from an atomic clip $j$, i.e. $a_j \ket{\ldots, n_j, \ldots} \propto \ket{\ldots, n_j-1, \ldots}$, and is called an \emph{annihilation operator}. For the special case $n_j = 0$, we have $a_j \ket{\ldots, n_j, \ldots} = 0$.

A typical Hamiltonian $H$ in second quantization is of the form 
\begin{align} \label{Equation:GeneralHamiltonian}
    H = \sum_{o,i} \sum_{j_1, \ldots , j_i} \sum_{k_1, \ldots , k_o}& h(j_1, \ldots , j_i, k_1,\ldots , k_o )\, a^\dagger_{k_1} \ldots  a^\dagger_{k_o} \cdot a_{j_1} \ldots a_{j_i}
\end{align}
with $h(j_1, \ldots , j_i, k_1,\ldots , k_o ) \in \mathbb C$. In most cases, $o$ and $i$ take very small values (smaller than 10). A commonly used ansatz is of the form 
\begin{equation}
    H = \sum_{j,k} \epsilon_{j,k} a^\dagger_k a_j + \sum_{j_1, j_2, k_1, k_2} V_{j_1, j_2, k_1, k_2} a^\dagger_{k_1} a^\dagger_{k_2} a_{j_1} a_{j_2}\, ,\label{Equation:TwoBodyHamiltonian} 
\end{equation}
where $\epsilon_{j,k}, V_{j_1, j_2, k_1, k_2} \in \mathbb C$. The second term in Eq. \eqref{Equation:TwoBodyHamiltonian} is called a \emph{two-body interaction} because this interaction involves two ingoing and two outgoing excitations. 

For small enough time intervals $\delta t$, the time evolution operator can be approximated as $e^{i H \delta t} \approx \mathbb 1 + i H \delta t$. We discard the identity operator $\mathbb 1$ (in physical terms, we \emph{post-select} on a non-trivial change occurring) because it means that no transition occurred at all. We discretize time in multiples of $\delta t$, and identify each step in the random walk with one application of $i \delta t H$, absorbing $i\delta t$ into the coefficients $h(j_1, \ldots , j_i, k_1,\ldots , k_o )$. 

In many-body physics, given a state of the form $\sum_{n_1, \ldots n_{|V|}} \alpha_{n_1, \ldots , n_{|V|}} \ket{n_1, \ldots , n_{|V|}}$ with $\alpha_{n_1, \ldots , n_{|V|}} \in \mathbb C$ and $\sum_{n_1, \ldots, n_{|V|}} |\alpha_{n_1, \ldots , n_{|V|}}|^2 = 1$, each $|\alpha_{n_1, \ldots , n_{|V|}}|^2$ gives the probability to find the many-body system in excitation configuration $\ket{n_1, \ldots , n_{|V|}}$. If we start from a state $\ket{n_1, \ldots , n_{|V|}}$, under our assumptions, the next state is essentially $H \ket{n_1, \ldots , n_{|V|}}$. Therefore, the transition probabilities are essentially the (modulus square) of the entries of $H$.

For our many-body physics-inspired inductive bias, we identify each term $|h(j_1, \ldots , j_i, k_1,\ldots , k_o )|^2$ from $H$ with an unnormalized transition probability/$h$-value $h_{(i, o)}(\{c_{j_1}, \ldots, c_{j_i} \}, \{ c_{k_1}, \ldots, c_{k_o} \}) \in \mathbb R$. Here, the set symbol $\{ \}$ indicates that the order of atomic clips does not matter (in many-body physics, all $a_j$ commute or anti-commute with each other). A crucial part of our inductive bias is that we demand a cutoff for $i$ and $o$. One can also introduce further physics-inspired assumptions such as particle number conservation, formalized as $i = o$. 

The last ingredient in our inductive bias is the observation that the operators $h(j_1, \ldots , j_i, k_1,\ldots , k_o ) a^\dagger_{k_1} \cdots  a^\dagger_{k_o} \cdot a_{j_1} \cdots a_{j_i}$ also act on excitation configurations $\ket{n_1, \ldots , n_{|V|}}$ that have excitations $n_x = 1$ for $x \ne j_1, \ldots, j_i, k_1, \ldots, k_o$, but leave $n_x$ unchanged. Similarly, we also apply $h$-values $h_{(i, o)}(\{c_{j_1}, \ldots, c_{j_i} \}, \{ c_{k_1}, \ldots, c_{k_o} \})$ to excitation configurations that have additional excitations in unrelated atomic clips $c_x$.

\subsection{First Steps Towards Quantum $\text{MEPS}$} \label{Section:QuantumMEPS}

MEPS and the many-body inductive biases are classical machine learning methods mimicking quantum many-body systems. Therefore, it is natural to consider quantum-mechanical MEPS agents implemented on quantum hardware which uses engineered quantum walks of physical particles. Examples of such quantum hardware include certain kinds of quantum simulators~\cite{Simulators1, Simulators2} and integrated photonics chips~\cite{Photonics, Flamini_Krumm2023}.

As we described in Section \ref{Section:Motivation}, the dynamics of such quantum systems are described by time evolution operators of the form $e^{-i t H}$ (or more generally, $\mathcal T e^{-i\int_{t_1}^{t_2} H(t) \mathrm{d}t}$, where $\mathcal T$ is the time-ordering operator). 

We emphasize that most of the considerations in Section~\ref{Section:Motivation} served to formally develop a classical and discrete analogue. On quantum simulation hardware, we can directly define a parametrized Hamiltonian $H$ of the form in Eq.~\eqref{Equation:GeneralHamiltonian} and define the deliberation as its time evolution, which is a continuous time quantum walk.

The coefficients $h(\{j_1, \ldots, j_i\}, \{k_1,\ldots, k_o\}) \in \mathbb C$ in Eq. \eqref{Equation:GeneralHamiltonian} are then the trainable parameters or functions of the trainable parameters (such as tunable tunnelling amplitudes and 2-body interaction couplings). 

To couple in percepts, one would inject excitations into the corresponding percept modes. Meanwhile, to couple out actions, one would continuously measure the excitation number of action modes (e.g. using stroboscopic measurements) until they meet a condition for coupling out certain actions. The intermediate modes would remain unobserved such that time evolution can be coherent; this would correspond to the internal deliberative process of the agent.

One issue to consider is that physical Hamiltonians $H$ are Hermitian, i.e. $H^\dagger = H$. This has the consequence that $h(\{j_1, \ldots, j_i\}, \{k_1,\ldots, k_o\}) = \overline{h( \{k_1,\ldots, k_o\} , \{j_1, \ldots, j_i\}, )}$, meaning that the transition amplitude to go forward is just as large as the amplitude to go backward. This issue is already present in the quantization of basic PS and was addressed in ~\cite{Briegel_Cuevas2012} by using dissipation (or other irreversible, open quantum system evolutions) to obtain a broader class of parametrized time evolutions.

Another approach is to introduce an extra semi-classical degree of freedom that breaks the symmetry by acting as a clock. One inspiration for how to do this comes from integrated interferometer chips. Here, the photons have a definite velocity in the lateral direction of the photonics chip, but perform quantum walks in the transversal direction~\cite{Photonics}. This approach was applied in a recent proposal for a quantum PS with photons~\cite{Flamini_Krumm2023}.

It was mentioned in Remark \ref{Remark:Occupied} that our choice in classical MEPS about two excitations in the same atomic clip does not faithfully dequantize the behavior of any quantum particles. It would be interesting to analyze how the decision-making process is influenced by also being physically faithful in this regard. For example, the phenomenon of Pauli pressure in fermions could prevent an already excited atomic clip from being excited again, potentially putting the second excitation to better use somewhere else. This has the consequence that the random walks of fermions depend on each other, even if the Hamiltonian has no interaction terms.

\section{Discussion and Outlook} \label{Section:Conclusions}

In this paper, we introduced an XAI framework called MEPS, which allows us to model chains of thoughts as random walks of several particles on a hypergraph. The use of several particles allows for the representation of thoughts relying on the combination of several elementary concepts simultaneously, revealing and exploiting the composite structure of thoughts and thus greatly improving model interpretability. This added flexibility is a stepping stone in developing systematic methods that let us model domain knowledge via the structure of the hypergraph and attach concepts to clusters of relevant vertices on the hypergraph. A new definition for dynamic hypergraph was also introduced to model the agent's training history and serve as a tool for explainability by connecting with the hypergraph visualization literature.

To reduce the exponential complexity of a naive implementation of MEPS to a low-degree polynomial complexity, we defined an inductive bias. This inductive bias is a classical analogue of the time evolution of quantum many-body systems. This inductive bias includes a cutoff regarding how many particles can participate in an interaction. We proved that our inductive bias indeed leads to a polynomial complexity, with the degree given by the interaction cutoff. We believe that our inductive bias does not severely restrict the potential of MEPS agents in many scenarios. This belief is motivated by the fact that humans can also only combine a handful of concepts simultaneously. Nonetheless, humans display an unmatched ability to quickly adapt to a wide range of environments. 

The explainability of MEPS and the power of our inductive bias were demonstrated in three synthetic environments: two extensions of the Invasion Game and a broken computer diagnosis and repair scenario. The Invasion Game modifications were chosen to visualize in a clean and simple setting the impact of an appropriate many-body inductive bias on the learning process. Using less than necessary excitations leads to bad returns (``underfitting''), while using extra unnecessary excitations slows down the training. In the Computer Maintenance setting, we used a multi-layered agent to provide chains-of-thought of length 2, where we distinguished between the agent's belief about the causes of problems and the fixes necessary to solve those problems. Hypothesis and plausibility rewards were introduced and applied to different segments of the deliberation path to overcome the credit assignment problem, which tends to occur when blindly applying the basic PS update rule to intermediate layers. The structure of the plausibility reward in particular, encoded causal elements that offered a mechanism for the agent to weakly generalize to unseen percepts. The multi-layer architecture combined with the reward structure helped demonstrate the ease with which a MEPS agent can successfully navigate a complex, real-world-inspired environment while maintaining explainability. The inductive bias also proved useful in greatly cutting down the total number of steps and model trainable parameters required to reach the optimal policy, in particular, compared to non-PS algorithms such as Q-learning and its tabular PS-inspired multi-layer extension introduced in this work.

At last, we presented basic approaches for how to develop a quantization of our classical MEPS and the inductive bias suitable for actual quantum computers, focusing on near-term quantum simulators and integrated photonics hardware. In particular, we reviewed obstacles one will encounter, such as hermiticity/unitarity constraints of the time evolution operators, and potential mitigation strategies.

There are many avenues to build upon our work, with the most obvious one being an application of our methodology to other types of learning settings. A fruitful strategy might be to look at the behavioral biology examples considered in previous PS literature~\cite{Bees, Foraging}, where multiple excitations can be used to explicitly represent different concepts that matter to the animal or agent to make a decision, like the presence of pheromones and threats, or the creation of a mental map of the environment. Similarly, MEPS could be used to model phenomena from psychology, such as the behaviour of a cat modelled in \cite{Muller_Briegel2018}, which one would not consider in typical machine learning settings.

Quantum many-body systems have additional properties which we did not consider in this work. One important such property is that particles can only directly interact if they are physically close. Particles classified as fermions tend to avoid each other, and the related mechanisms are known under names such as \emph{Pauli-exclusion}, \emph{Pauli-pressure}, and \emph{excluded volume}. Nonetheless, local elementary interactions of few particles give rise to most phenomena known in physics. Therefore, it might be worthwhile to model an additional inductive bias which formalizes a notion of distance between atomic clips, and restricts $h^{(n)}_{(i,o)}$ to only be non-zero for close excitations. 

Another important property of quantum many-body systems concerns the question of what happens if one tries to put an excitation on an already excited atomic clip. As explained in Remark~\ref{Remark:Occupied}, we made an unphysical choice motivated by the interpretation of atomic clips as concepts. Therefore, it would be interesting to investigate what would happen if one instead mimicked physics in this aspect. For fermions, putting an excitation in an already excited atomic clip would result in an empty atomic clip. For bosons, several excitations could be in the same atomic clip, meaning one would first have to model an extension of MEPS that allows several excitations on the same atomic clip.    

The binary nature of our excitation configurations suggests the existence of potential relations to the field of Neuro-Symbolic (NeSy) logic~\cite{NeSy1, NeSy2}. Indeed, one can read the presence of an excitation on an atomic clip as a truth state that the concept represented by the atomic clip is currently relevant or applicable. Investigating these relations might lead to a fruitful cross-fertilization between the two fields. However, it is also important to emphasize the differences. An important ambition of our MEPS scheme is to model various chain-of-thought processes, not just those relying on formal logic to process facts of the environment. These also include thought processes that underly irrational or bounded rational decisions, as studied in psychology and the decision sciences. Furthermore, MEPS comes with a natural update rule based on h-values that could also be applied to NeSy.

On the numerical side, compiling the for-loops for or partially parallelizing the deliberation process, or using advanced Monte Carlo Simulation~\cite{MonteCarlo1, MonteCarlo2} software might significantly speed up MEPS agents. For Python implementations, a first route towards this goal might be to use Cython~\cite{Cython} or just-in-time compiling modules like Numba~\cite{lam2015numba}. To numerically profit from our inductive biases, it is important to sample transitions in a way that does not require iterating over the full power set of atomic clips. We already formulated one method for sampling transitions, but we have no guarantee that it is the best possible implementation. Indeed, we appealed very little to results from the mathematical literature on hypergraphs such as hypergraph expansion techniques (star, cluster, line, etc.), Laplacian spectral clustering techniques, factorization of hypergraph matrix representations, and other hierarchical partitionings of hypergraphs in developing MEPS ~\cite{GraphBook1, GraphBook2, Dai_Gao2023}. We also assumed the atomic clips within each hyperedge were of equal importance, which may not be the case in situations where some information or properties of the data are privileged over others, so incorporating atomic clip weights into the training process could be beneficial in this regard ~\cite{Dai_Gao2023}. Further exploration into the hypergraph literature will likely yield many improvements for our MEPS methodology and should be considered as an important next step in the development of MEPS; chiefly for the scalability of the model. 

Furthermore, we employed a dual reward mechanism to avoid the credit assignment problem but we did not explore other potential mitigation strategies that could alleviate this problem. A search for systematic methods to find good initialization strategies, or to adapt Imitation Learning methods to our setting~\cite{ImitationSurvey1, ImitationSurvey2} might prove useful.

\section{Acknowledgements}
This research was funded in part by the Austrian Science Fund (FWF) [10.55776/F7102, 10.55776/WIT9503323]. For open access purposes, the authors have applied a CC BY public copyright license to any author-accepted manuscript version arising from this submission.
We gratefully acknowledge support from the European Union (ERC Advanced Grant, QuantAI, No. 101055129). The views and opinions expressed in this article are however those of the author(s) only and do not necessarily reflect those of the European Union or the European Research Council - neither the European Union nor the granting authority can be held responsible for them.

\bibliographystyle{elsarticle-num} 
\bibliography{MEPS}

@article{Briegel_Cuevas2012,
  title = {Proective simulation for artificial intelligence},
  author = {Hans J. Briegel and Gemma de las Cuevas},
  journal = {Sci. Rep.},
  volume = {2},
  pages = {400},
  year = {2012},
  publisher = {Nature},
  doi = {https://doi.org/10.1038/srep00400}
}

@article{Flamini_Krumm2023,
	title = {Towards interpretable quantum machine learning via single-photon quantum walks},
	url = {http://arxiv.org/abs/2301.13669},
	doi = {10.48550/arXiv.2301.13669},
	publisher = {arXiv},
	author = {Flamini, Fulvio and Krumm, Marius and Fiderer, Lukas J. and Müller, Thomas and Briegel, Hans J.},
	year = {2023},
	journal = {arXiv:2301.13669}
}

@article{Mautner_Makmal2015,
	title = {Projective {Simulation} for {Classical} {Learning} {Agents}: {A} {Comprehensive} {Investigation}},
	volume = {33},
	issn = {1882-7055},
	url = {https://doi.org/10.1007/s00354-015-0102-0},
	doi = {10.1007/s00354-015-0102-0},
	number = {1},
	journal = {New Gener. Comput.},
	author = {Mautner, Julian and Makmal, Adi and Manzano, Daniel and Tiersch, Markus and Briegel, Hans J.},
	year = {2015},
	pages = {69--114},
}

@article{OptimalPS,
  title={On the convergence of projective-simulation--based reinforcement learning in Markov decision processes},
  author={Boyajian, Walter L. and Clausen, Jens and Trenkwalder, Lea M. and Dunjko, Vedran and Briegel, Hans J.},
  journal={Quantum machine intelligence},
  volume={2},
  pages={13},
  year={2020},
  publisher={Springer}
}

@article{LiuJinRobust,
title = {A comprehensive survey of robust deep learning in computer vision},
journal = {{Journal of Automation and Intelligence}},
year = {2023},
volume = {(In Press, Corrected Proof)},
issn = {2949-8554},
doi = {https://doi.org/10.1016/j.jai.2023.10.002},
url = {https://www.sciencedirect.com/science/article/pii/S294985542300045X},
author = {Jia Liu and Yaochu Jin},
}

@misc{ImitationSurvey1,
  title={A Survey of Imitation Learning: Algorithms, Recent Developments, and Challenges},
  author={Zare, Maryam and Kebria, Parham M. and Khosravi, Abbas and Nahavandi, Saeid},
  howpublished={arXiv preprint arXiv:2309.02473},
  year={2023}
}

@article{ImitationSurvey2,
  title={Imitation learning: A survey of learning methods},
  author={Hussein, Ahmed and Gaber, Mohamed Medhat and Elyan, Eyad and Jayne, Chrisina},
  journal={ACM Computing Surveys (CSUR)},
  volume={50},
  number={2},
  pages={1--35},
  year={2017},
  publisher={ACM New York, NY, USA}
}

@article{XAI1,
  title={A systematic review of Explainable Artificial Intelligence models and applications: Recent developments and future trends},
  author={Saranya, A. and Subhashini, R.},
  journal={{Decision Analytics Journal}},
  pages={100230},
  volume = {7},
  year={2023},
  publisher={Elsevier}
}

@article{XAI2,
  title={{Survey on Explainable AI: From Approaches, Limitations and Applications Aspects}},
  author={Yang, Wenli and Wei, Yuchen and Wei, Hanyu and Chen, Yanyu and Huang, Guan and Li, Xiang and Li, Renjie and Yao, Naimeng and Wang, Xinyi and Gu, Xiaotong and others},
  journal={Human-Centric Intelligent Systems},
  volume={3},
  number={3},
  pages={161--188},
  year={2023},
  publisher={Springer}
}

@book{QFT1,
    author = "Peskin, Michael E. and Schroeder, Daniel V.",
    title = "{An Introduction to quantum field theory}",
    isbn = "978-0-201-50397-5",
    publisher = "Addison-Wesley",
    address = "Reading, USA",
    year = "1995"
}

@book{QFT2,
  title={Quantum field theory and the standard model},
  author={Schwartz, Matthew D.},
  year={2014},
  publisher={Cambridge university press}
}

@book{QFT3,
  title={Nuclear and particle physics: an introduction},
  author={Martin, Brian R. and Shaw, Graham},
  year={2019},
  publisher={John Wiley \& Sons}
}

@book{ManyBody1,
  title={A Course in Quantum Many-body Theory: From Conventional Fermi Liquids to Strongly Correlated Systems},
  author={Fabrizio, Michele},
  year={2022},
  publisher={Springer Nature}
}

@book{ManyBody2,
  title={Introduction to many-body physics},
  author={Coleman, Piers},
  year={2015},
  publisher={Cambridge University Press}
}

@article{MLinPhysics,
  title={Machine learning and the physical sciences},
  author={Carleo, Giuseppe and Cirac, Ignacio and Cranmer, Kyle and Daudet, Laurent and Schuld, Maria and Tishby, Naftali and Vogt-Maranto, Leslie and Zdeborov{\'a}, Lenka},
  journal={Reviews of Modern Physics},
  volume={91},
  number={4},
  pages={045002},
  year={2019},
  publisher={APS}
}

@article{Medicine,
  title={Medical deep learning—A systematic meta-review},
  author={Egger, Jan and Gsaxner, Christina and Pepe, Antonio and Pomykala, Kelsey L and Jonske, Frederic and Kurz, Manuel and Li, Jianning and Kleesiek, Jens},
  journal={Computer methods and programs in biomedicine},
  volume={221},
  pages={106874},
  year={2022},
  publisher={Elsevier}
}

@inproceedings{LLM1,
title={Chain of Thought Prompting Elicits Reasoning in Large Language Models},
author={Jason Wei and Xuezhi Wang and Dale Schuurmans and Maarten Bosma and Brian Ichter and Fei Xia and Ed H. Chi and Quoc V Le and Denny Zhou},
booktitle={Advances in Neural Information Processing Systems},
editor={Alice H. Oh and Alekh Agarwal and Danielle Belgrave and Kyunghyun Cho},
year={2022},
url={https://openreview.net/forum?id=_VjQlMeSB_J}
}

@misc{LLM2,
      title={{When do you need Chain-of-Thought Prompting for ChatGPT?}}, 
      author={Jiuhai Chen and Lichang Chen and Heng Huang and Tianyi Zhou},
      year={2023},
      eprint={2304.03262},
      archivePrefix={arXiv},
      primaryClass={cs.AI}
}

@Book{DeepLearningBook,
  Title                    = {Deep Learning},
  Author                   = {Ian J. Goodfellow and Yoshua Bengio and Aaron Courville},
  Publisher                = {MIT Press},
  Year                     = {2016},

  Address                  = {Cambridge, MA, USA},
  Note                     = {\url{http://www.deeplearningbook.org}}
}

@Article{Adversarial1,
AUTHOR = {Qiu, Shilin and Liu, Qihe and Zhou, Shijie and Wu, Chunjiang},
TITLE = {Review of Artificial Intelligence Adversarial Attack and Defense Technologies},
JOURNAL = {Applied Sciences},
VOLUME = {9},
YEAR = {2019},
NUMBER = {5},
ARTICLE-NUMBER = {909},
URL = {https://www.mdpi.com/2076-3417/9/5/909},
ISSN = {2076-3417},
}

@article{Adversarial2,
author = {Yao Li, Minhao Cheng, Cho-Jui Hsieh and Thomas C. M. Lee},
title = {A Review of Adversarial Attack and Defense for Classification Methods},
journal = {The American Statistician},
volume = {76},
number = {4},
pages = {329-345},
year = {2022},
publisher = {Taylor & Francis},
doi = {10.1080/00031305.2021.2006781},
URL = {https://doi.org/10.1080/00031305.2021.2006781},
eprint = {https://doi.org/10.1080/00031305.2021.2006781}
}

@book{GraphBook1,
author = {Alain Bretto},
title = {Hypergraph Theory - An Introduction},
year = {2013},
publisher = {Springer Cham}
}

@book{GraphBook2,
  title={Graphs and Hypergraphs},
  author={Berge, Claude},
  isbn={9780444103994},
  lccn={lc72088288},
  publisher={North-Holland Publishing Company},
  year = {1973}
}

@book{Dai_Gao2023,
author = {Qionghai Dai and Yue Gao},
title = {Hypergraph Computation},
year = {2023},
publisher = {Springer}
}

@book{MonteCarlo1,
author = {Adrian Barbu and Song-Chun Zhu},
title = {Monte Carlo Methods},
year = {2020},
publisher = {Springer Singapore}
}

@book{MonteCarlo2,
  title={Hands-on Markov models with python: Implement probabilistic models for learning complex data sequences using the Python ecosystem},
  author={Ankan, Ankur and Panda, Abinash},
  year={2018},
  publisher={Packt Publishing Ltd}
}

@article{Foraging,
  title={Optimal foraging strategies can be learned},
  author={Mu{\~n}oz-Gil, Gorka and L{\'o}pez-Incera, Andrea and Fiderer, Lukas J. and Briegel, Hans J.},
  journal={New Journal of Physics},
  volume={26},
  pages={013010},
  year={2024}
}

@article{Bees,
author={López-Incera,Andrea and Nouvian,Morgane and Ried,Katja and Müller,Thomas and Briegel,Hans J.},
year={2021},
title={Honeybee communication during collective defence is shaped by predation},
journal={BMC Biology},
volume={19},
pages={1-16},
url={https://www.proquest.com/scholarly-journals/honeybee-communication-during-collective-defence/docview/2543446022/se-2}
}

@article{Robotics,
  title={Skill learning by autonomous robotic playing using active learning and exploratory behavior composition},
  author={Hangl, Simon and Dunjko, Vedran and Briegel, Hans J. and Piater, Justus},
  journal={Frontiers in Robotics and AI},
  volume={7},
  pages={42},
  year={2020},
  publisher={Frontiers Media SA}
}

@article{Navigation,
  title={Benchmarking projective simulation in navigation problems},
  author={Melnikov, Alexey A. and Makmal, Adi and Briegel, Hans J.},
  journal={IEEE Access},
  volume={6},
  pages={64639--64648},
  year={2018},
  publisher={IEEE}
}

@article{PSforQuantumExperiments,
  title={Active learning machine learns to create new quantum experiments},
  author={Melnikov, Alexey A and Poulsen Nautrup, Hendrik and Krenn, Mario and Dunjko, Vedran and Tiersch, Markus and Zeilinger, Anton and Briegel, Hans J.},
  journal={Proceedings of the National Academy of Sciences},
  volume={115},
  number={6},
  pages={1221--1226},
  year={2018},
  publisher={National Acad Sciences}
}

@book{RLbook,
  title={Reinforcement learning: An introduction},
  author={Sutton, Richard S. and Barto, Andrew G.},
  year={2018},
  publisher={MIT press}
}

@INPROCEEDINGS{DeepLearningApplications,
  author={Shinde, Pramila P. and Shah, Seema},
  booktitle={2018 Fourth International Conference on Computing Communication Control and Automation (ICCUBEA)}, 
  title={A Review of Machine Learning and Deep Learning Applications}, 
  year={2018},
  volume={},
  number={},
  pages={1-6},
  doi={10.1109/ICCUBEA.2018.8697857}}

@ARTICLE{Cython,
    author={Behnel, S. and Bradshaw, R. and Citro, C. and Dalcin, L. and Seljebotn, D.S. and Smith, K.},
    journal={Computing in Science Engineering},
    title={Cython: The Best of Both Worlds},
    year={2011},
    volume={13},
    number={2},
    pages={31 -39},
    doi={10.1109/MCSE.2010.118},
    ISSN={1521-9615},
}

@misc{Bandits1,
  title={A survey on contextual multi-armed bandits},
  author={Zhou, Li},
  howpublished={arXiv preprint arXiv:1508.03326},
  year={2015}
}

@inproceedings{Bandits2,
  title={Survey on applications of multi-armed and contextual bandits},
  author={Bouneffouf, Djallel and Rish, Irina and Aggarwal, Charu},
  booktitle={2020 IEEE Congress on Evolutionary Computation (CEC)},
  pages={1--8},
  year={2020},
  organization={IEEE}
}

@article{Photonics,
  title={Photonic quantum information processing: a review},
  author={Flamini, Fulvio and Spagnolo, Nicolo and Sciarrino, Fabio},
  journal={Reports on Progress in Physics},
  volume={82},
  number={1},
  pages={016001},
  year={2018},
  publisher={IOP Publishing}
}

@article{InductiveBias1,
  title={Inductive biases for deep learning of higher-level cognition},
  author={Goyal, Anirudh and Bengio, Yoshua},
  journal={Proceedings of the Royal Society A},
  volume={478},
  number={2266},
  pages={20210068},
  year={2022},
  publisher={The Royal Society}
}

@book{InductiveBias2,
  title={Inductive Biases in Machine Learning for Robotics and Control},
  author={Lutter, Michael},
  volume={156},
  year={2023},
  publisher={Springer Nature}
}

@phdthesis{InductiveBias3,
  title={Inductive Bias in Machine Learning},
  author={Rendsburg, Luca Silvester},
  year={2022},
  school={Eberhard Karls Universit{\"a}t T{\"u}bingen},
  url = {https://d-nb.info/1280233419/34}
}

@inproceedings{Modularity1,
  title={Modularity in Deep Learning: A Survey},
  author={Sun, Haozhe and Guyon, Isabelle},
  booktitle={Science and Information Conference},
  pages={561--595},
  year={2023},
  organization={Springer}
}

@article{Modularity2,
  title={A review of modularization techniques in artificial neural networks},
  author={Amer, Mohammed and Maul, Tom{\'a}s},
  journal={Artificial Intelligence Review},
  volume={52},
  pages={527--561},
  year={2019},
  publisher={Springer}
}

@misc{Causal1,
    title={Causal Machine Learning: A Survey and Open Problems},
    author={Jean Kaddour and Aengus Lynch and Qi Liu and Matt J. Kusner and Ricardo Silva},
    year={2022},
    url = {https://arxiv.org/abs/2206.15475},
    howpublished={arXiv preprint arXiv:2206.15475},
}

@misc{Causal2,
  title={Causal Reinforcement Learning: A Survey},
  author={Deng, Zhihong and Jiang, Jing and Long, Guodong and Zhang, Chengqi},
  howpublished={arXiv preprint arXiv:2307.01452},
  year={2023}
}

@INPROCEEDINGS{CNN1,
  author={Ajit, Arohan and Acharya, Koustav and Samanta, Abhishek},
  booktitle={2020 International Conference on Emerging Trends in Information Technology and Engineering (ic-ETITE)}, 
  title={A Review of Convolutional Neural Networks}, 
  year={2020},
  volume={},
  number={},
  pages={1-5},
  doi={10.1109/ic-ETITE47903.2020.049}}

@article{CNN2,
  title={A review of convolutional neural network architectures and their optimizations},
  author={Cong, Shuang and Zhou, Yang},
  journal={Artificial Intelligence Review},
  volume={56},
  number={3},
  pages={1905--1969},
  year={2023},
  publisher={Springer}
}

@article{NeSy1,
title = {A survey on neural-symbolic learning systems},
journal = {Neural Networks},
volume = {166},
pages = {105-126},
year = {2023},
issn = {0893-6080},
doi = {https://doi.org/10.1016/j.neunet.2023.06.028},
url = {https://www.sciencedirect.com/science/article/pii/S0893608023003398},
author = {Dongran Yu and Bo Yang and Dayou Liu and Hui Wang and Shirui Pan},
}

@misc{NeSy2,
  title={Neurosymbolic AI and its Taxonomy: a survey},
  author={Gibaut, Wandemberg and Pereira, Leonardo and Grassiotto, Fabio and Osorio, Alexandre and Gadioli, Eder and Munoz, Amparo and Gomes, Sildolfo and Santos, Claudio dos},
  howpublished={arXiv preprint arXiv:2305.08876},
  year={2023}
}

@article{Simulators1,
  title = {Quantum Simulators: Architectures and Opportunities},
  author = {Altman, Ehud and Brown, Kenneth R. and Carleo, Giuseppe and Carr, Lincoln D. and Demler, Eugene and Chin, Cheng and DeMarco, Brian and Economou, Sophia E. and Eriksson, Mark A. and Fu, Kai-Mei C. and Greiner, Markus and Hazzard, Kaden R.A. and Hulet, Randall G. and Koll\'ar, Alicia J. and Lev, Benjamin L. and Lukin, Mikhail D. and Ma, Ruichao and Mi, Xiao and Misra, Shashank and Monroe, Christopher and Murch, Kater and Nazario, Zaira and Ni, Kang-Kuen and Potter, Andrew C. and Roushan, Pedram and Saffman, Mark and Schleier-Smith, Monika and Siddiqi, Irfan and Simmonds, Raymond and Singh, Meenakshi and Spielman, I.B. and Temme, Kristan and Weiss, David S. and Vu\ifmmode \check{c}\else \v{c}\fi{}kovi\ifmmode \acute{c}\else \'{c}\fi{}, Jelena and Vuleti\ifmmode \acute{c}\else \'{c}\fi{}, Vladan and Ye, Jun and Zwierlein, Martin},
  journal = {PRX Quantum},
  volume = {2},
  issue = {1},
  pages = {017003},
  numpages = {19},
  year = {2021},
  month = {Feb},
  publisher = {American Physical Society},
  doi = {10.1103/PRXQuantum.2.017003},
  url = {https://link.aps.org/doi/10.1103/PRXQuantum.2.017003}
}

@article{Simulators2,
  title = {Quantum simulation},
  author = {Georgescu, I. M. and Ashhab, S. and Nori, Franco},
  journal = {Rev. Mod. Phys.},
  volume = {86},
  issue = {1},
  pages = {153--185},
  numpages = {33},
  year = {2014},
  month = {Mar},
  publisher = {American Physical Society},
  doi = {10.1103/RevModPhys.86.153},
  url = {https://link.aps.org/doi/10.1103/RevModPhys.86.153}
}

@misc{GitHub,
    author = {Philip A. LeMaitre and Marius Krumm},
    howpublished = {\url{https://github.com/MariusKrumm/ManyBodyMEPS}}
}

@inproceedings{lam2015numba,
  title={Numba: A llvm-based python jit compiler},
  author={Lam, Siu Kwan and Pitrou, Antoine and Seibert, Stanley},
  booktitle={Proceedings of the Second Workshop on the LLVM Compiler Infrastructure in HPC},
  pages={1--6},
  year={2015}
}

@article{Vehlow_Beck_Weiskopf2016,
  author={Vehlow, Corinna and Beck, Fabian and Weiskopf, Daniel},
  journal={IEEE Transactions on Visualization and Computer Graphics}, 
  title={Visualizing Dynamic Hierarchies in Graph Sequences}, 
  year={2016},
  volume={22},
  number={10},
  pages={2343-2357},
  doi={10.1109/TVCG.2015.2507595}}

@article{Muller_Briegel2018,
  author={M\"uller, Thomas and Briegel, Hans J.},
  journal={Dialectica}, 
  title={A Stochastic Process Model for Free Agency under Indeterminism}, 
  year={2018},
  volume={72},
  number={2},
  pages={219-252}
}

@inproceedings{fischer_towards_2021,
	title = {Towards a {Survey} on {Static} and {Dynamic} {Hypergraph} {Visualizations}},
	url = {http://arxiv.org/abs/2107.13936},
	doi = {10.1109/VIS49827.2021.9623305},
	urldate = {2024-02-08},
	booktitle = {2021 {IEEE} {Visualization} {Conference} ({VIS})},
	author = {Fischer, Maximilian T. and Frings, Alexander and Keim, Daniel A. and Seebacher, Daniel},
	month = oct,
	year = {2021},
	note = {arXiv:2107.13936 [cs]},
	pages = {81--85},
}

@article{fischer_visual_2021,
	title = {Visual {Analytics} for {Temporal} {Hypergraph} {Model} {Exploration}},
	volume = {27},
	issn = {1077-2626, 1941-0506, 2160-9306},
	url = {http://arxiv.org/abs/2008.07299},
	doi = {10.1109/TVCG.2020.3030408},
	number = {2},
	urldate = {2024-02-08},
	journal = {IEEE Transactions on Visualization and Computer Graphics},
	author = {Fischer, Maximilian T. and Arya, Devanshu and Streeb, Dirk and Seebacher, Daniel and Keim, Daniel A. and Worring, Marcel},
	month = feb,
	year = {2021},
	note = {arXiv:2008.07299 [cs]},
	pages = {550--560},
}

@misc{feng_hypergraph_2019,
	title = {Hypergraph {Neural} {Networks}},
	url = {http://arxiv.org/abs/1809.09401},
	urldate = {2024-02-08},
	publisher = {arXiv},
	author = {Feng, Yifan and You, Haoxuan and Zhang, Zizhao and Ji, Rongrong and Gao, Yue},
	month = feb,
	year = {2019},
	note = {arXiv:1809.09401 [cs, stat]},
}

@misc{yadati_hypergcn_2019,
	title = {{HyperGCN}: {A} {New} {Method} of {Training} {Graph} {Convolutional} {Networks} on {Hypergraphs}},
	shorttitle = {{HyperGCN}},
	url = {http://arxiv.org/abs/1809.02589},
	urldate = {2024-02-08},
	publisher = {arXiv},
	author = {Yadati, Naganand and Nimishakavi, Madhav and Yadav, Prateek and Nitin, Vikram and Louis, Anand and Talukdar, Partha},
	month = may,
	year = {2019},
	note = {arXiv:1809.02589 [cs, stat]},
}

@article{beck_taxonomy_2017,
	title = {A {Taxonomy} and {Survey} of {Dynamic} {Graph} {Visualization}},
	volume = {36},
	issn = {1467-8659},
	url = {https://onlinelibrary.wiley.com/doi/abs/10.1111/cgf.12791},
	doi = {10.1111/cgf.12791},
	number = {1},
	urldate = {2024-02-08},
	journal = {Computer Graphics Forum},
	author = {Beck, Fabian and Burch, Michael and Diehl, Stephan and Weiskopf, Daniel},
	year = {2017},
	note = {\_eprint: https://onlinelibrary.wiley.com/doi/pdf/10.1111/cgf.12791},
	pages = {133--159},
}

@article{TrustResponsible,
  title={Connecting the dots in trustworthy Artificial Intelligence: From AI principles, ethics, and key requirements to responsible AI systems and regulation},
  author={D{\'\i}az-Rodr{\'\i}guez, Natalia and Del Ser, Javier and Coeckelbergh, Mark and de Prado, Marcos L{\'o}pez and Herrera-Viedma, Enrique and Herrera, Francisco},
  journal={Information Fusion},
  volume={99},
  pages={101896},
  year={2023},
  publisher={Elsevier}
}

@article{XAI3,
  title={Explainable artificial intelligence: a comprehensive review},
  author={Minh, Dang and Wang, H Xiang and Li, Y Fen and Nguyen, Tan N},
  journal={Artificial Intelligence Review},
  pages={1--66},
  year={2022},
  publisher={Springer}
}

@article{XAI4,
  title={A comprehensive taxonomy for explainable artificial intelligence: a systematic survey of surveys on methods and concepts},
  author={Schwalbe, Gesina and Finzel, Bettina},
  journal={Data Mining and Knowledge Discovery},
  volume={38},
  number={5},
  pages={3043--3101},
  year={2024},
  publisher={Springer}
}

@article{XAI5,
  title={Explainable Artificial Intelligence (XAI): Concepts, taxonomies, opportunities and challenges toward responsible AI},
  author={Arrieta, Alejandro Barredo and D{\'\i}az-Rodr{\'\i}guez, Natalia and Del Ser, Javier and Bennetot, Adrien and Tabik, Siham and Barbado, Alberto and Garc{\'\i}a, Salvador and Gil-L{\'o}pez, Sergio and Molina, Daniel and Benjamins, Richard and others},
  journal={Information fusion},
  volume={58},
  pages={82--115},
  year={2020},
  publisher={Elsevier}
}

@article{XAI6,
  title={A practical tutorial on explainable AI techniques},
  author={Bennetot, Adrien and Donadello, Ivan and El Qadi El Haouari, Ayoub and Dragoni, Mauro and Frossard, Thomas and Wagner, Benedikt and Sarranti, Anna and Tulli, Silvia and Trocan, Maria and Chatila, Raja and others},
  journal={ACM Computing Surveys},
  volume={57},
  number={2},
  pages={1--44},
  year={2024},
  publisher={ACM New York, NY}
}

@article{XAITrustworthy,
  title={Explainable Artificial Intelligence (XAI): What we know and what is left to attain Trustworthy Artificial Intelligence},
  author={Ali, Sajid and Abuhmed, Tamer and El-Sappagh, Shaker and Muhammad, Khan and Alonso-Moral, Jose M and Confalonieri, Roberto and Guidotti, Riccardo and Del Ser, Javier and D{\'\i}az-Rodr{\'\i}guez, Natalia and Herrera, Francisco},
  journal={Information fusion},
  volume={99},
  pages={101805},
  year={2023},
  publisher={Elsevier}
}

@article{SafeReinforcementLearning,
  title={A comprehensive survey on safe reinforcement learning},
  author={Garc{\i}a, Javier and Fern{\'a}ndez, Fernando},
  journal={Journal of Machine Learning Research},
  volume={16},
  number={1},
  pages={1437--1480},
  year={2015}
}

@article{CounterfactualMedicine,
author = {Metsch, Jacqueline Michelle and Saranti, Anna and Angerschmid, Alessa and Pfeifer, Bastian and Klemt, Vanessa and Holzinger, Andreas and Hauschild, Anne-Christin},
title = {CLARUS: An interactive explainable AI platform for manual counterfactuals in graph neural networks },
year = {2024},
issue_date = {Feb 2024},
publisher = {Elsevier Science},
address = {San Diego, CA, USA},
volume = {150},
number = {C},
issn = {1532-0464},
url = {https://doi.org/10.1016/j.jbi.2024.104600},
doi = {10.1016/j.jbi.2024.104600},
journal = {J. of Biomedical Informatics},
month = feb,
numpages = {9}
}

@inproceedings{Shapley,
  title={The shapley value in machine learning},
  author={Rozemberczki, Benedek and Watson, Lauren and Bayer, P{\'e}ter and Yang, Hao-Tsung and Kiss, Oliv{\'e}r and Nilsson, Sebastian and Sarkar, Rik},
  booktitle={The 31st International Joint Conference on Artificial Intelligence and the 25th European Conference on Artificial Intelligence},
  pages={5572--5579},
  year={2022},
  organization={International Joint Conferences on Artificial Intelligence Organization}
}

@article{LIME,
title = {Local interpretable model-agnostic explanation approach for medical imaging analysis: A systematic literature review},
journal = {Computers in Biology and Medicine},
volume = {185},
pages = {109569},
year = {2025},
issn = {0010-4825},
doi = {https://doi.org/10.1016/j.compbiomed.2024.109569},
url = {https://www.sciencedirect.com/science/article/pii/S0010482524016548},
author = {Shahab Ul Hassan and Said Jadid Abdulkadir and M Soperi Mohd Zahid and Safwan Mahmood Al-Selwi},
}

@article{MechanisticInterpretability,
  title={Explaining AI through mechanistic interpretability},
  author={K{\"a}stner, Lena and Crook, Barnaby},
  journal={European Journal for Philosophy of Science},
  volume={14},
  number={4},
  pages={52},
  year={2024},
  publisher={Springer}
}

@article{AutomatedCircuitDiscovery,
  title={Towards automated circuit discovery for mechanistic interpretability},
  author={Conmy, Arthur and Mavor-Parker, Augustine and Lynch, Aengus and Heimersheim, Stefan and Garriga-Alonso, Adri{\`a}},
  journal={Advances in Neural Information Processing Systems},
  volume={36},
  pages={16318--16352},
  year={2023}
}

@article{DecisionTrees,
  title={Recent advances in decision trees: An updated survey},
  author={Costa, Vin{\'\i}cius G and Pedreira, Carlos E},
  journal={Artificial Intelligence Review},
  volume={56},
  number={5},
  pages={4765--4800},
  year={2023},
  publisher={Springer}
}

\appendix
\section*{Appendices}

\section{Additional Computer Maintenance Training Details} \label{Appendix:BrokenCompTrainDetails}

We can calculate the number of learning parameters $N_l$ for each agent from the expression 
\begin{align} \label{Equation:num_learn_param}
N_l &= \sum^{n_{hc}}_{k_2=1}\sum^{n_c}_{k_3=1}{N_c \choose k_2} {N_{ca} \choose k_3}\bigg(\sum^{n_s}_{k_1=1} {N_s \choose k_1}  \\
&\qquad \qquad \qquad + \sum^{n_{ac}}_{k_1=1}\sum^{n_f}_{k_4=1} {N_c \choose k_1} {N_f \choose k_4}\bigg) \nonumber \,,
\end{align}
where $N_s$ is the total number of symptoms; $N_c$, the components; $N_{ca}$, the causes; and $N_f$, the fixes. In this work, the values for each of these numbers are $N_s=8$, $N_c=5$, $N_{ca}=5$, and $N_f=4$.

The external reward shaping function $f(R; b)$ that is used in Subsection \ref{Subsection:ComputerMaintenance} has the form
\begin{align} \label{Equation:RewardShape}
f(R; b) = \max{\left(R - \theta(b)\frac{1}{4}\ln{(b+1)} , -16\right)}\,,
\end{align}
where $R$ is the reward, $\theta(x)$ is the step function, $b = a-a_{\mathrm{max}}$, $a$ is the current number of steps taken in an episode, and $a_{\mathrm{max}}$ is the maximum number of steps allowed before a penalty is applied. Values of $a_{\mathrm{max}}=500$ and $a_{\mathrm{max}}=1000$ were used in this work for the inductive bias and unrestricted agents, respectively.

\section{Modifications of the Inductive Bias}
\label{Appendix:OtherInductiveBiases}

In the main text, we mentioned two modifications of Inductive Bias~\ref{Bias:Layered} which guarantee low-order polynomial random walk lengths in the depth and width of the layered ECM. Now, we formulate their precise definition:

\begin{bias}{2SF and 2DP} \label{Bias:2Appendix}
    For weighted, layered feed-forward many-body hypergraphs with layers $(L_1,\dots , L_D)$, we introduce the following two modifications of Inductive Bias~\ref{Bias:ManyBody} and \ref{Bias:Layered}:

\begin{enumerate}
    \item[\textbf{SF.}] The \emph{ShallowFirst} (SF) Inductive Bias is the same as Inductive Bias \ref{Bias:Layered}, except for the following restriction:

    Let $\ell: V \rightarrow \{ 1,2, \dots D\}$ be the function that maps each atomic clip to the layer it is in. Then, given an excitation configuration $\{c_{m_1},\dots , c_{m_x}\}$, with the labelling such that the layers satisfy $\ell(c_{m_1}) = \cdots = \ell(c_{m_n}) < \ell(c_{m_{n+1}}) \le \cdots \le \ell(c_{m_x})$, the relevant $h_{(i,o)}(C_{\mathrm{in}}, C_{\mathrm{out}})$-values for this configuration are restricted to those with $C_{\mathrm{in}} \subset \{ c_{m_1}, \dots, c_{m_n} \}$.

    \item[\textbf{DP.}] The \emph{DiscardPassive} (DP) Inductive Bias is the same as Inductive Bias \ref{Bias:Layered}, except for the following modification:

    When performing a deliberation/random walk step on excitation configuration $\{c_{m_1}, \ldots , c_{m_x} \}$ with $h_{(i,o)}\big(\{c_{j_1}, \dots, c_{j_i} \}, \{c_{k_1}, \dots, c_{k_o}\}\big)$, all excitations except for those in $ \{c_{k_1}, \dots, c_{k_o}\}$ are discarded.
    \end{enumerate}

    Furthermore, we require that all random walks couple out an action if all excitations are in layer $L_D$, or earlier.
\end{bias}

In ~\ref{Appendix:RelationHypergraphs}, we will show that all inductive biases, including 2SF and 2DP, satisfy Definitions~\ref{Definition:MEPSarchitecture} and \ref{Definition:MEPSdynamics} when the standard probability assignment is used.

In ~\ref{Appendix:Complexity}, we will prove the promised upper bounds. 

Now, to demonstrate how these inductive biases work, we revisit Example~\ref{Example:InductiveBiases}, and this time include Inductive Biases 2SF and 2DP.

\begin{figure}[h!]
\centering
\includegraphics[width = 0.5\textwidth]{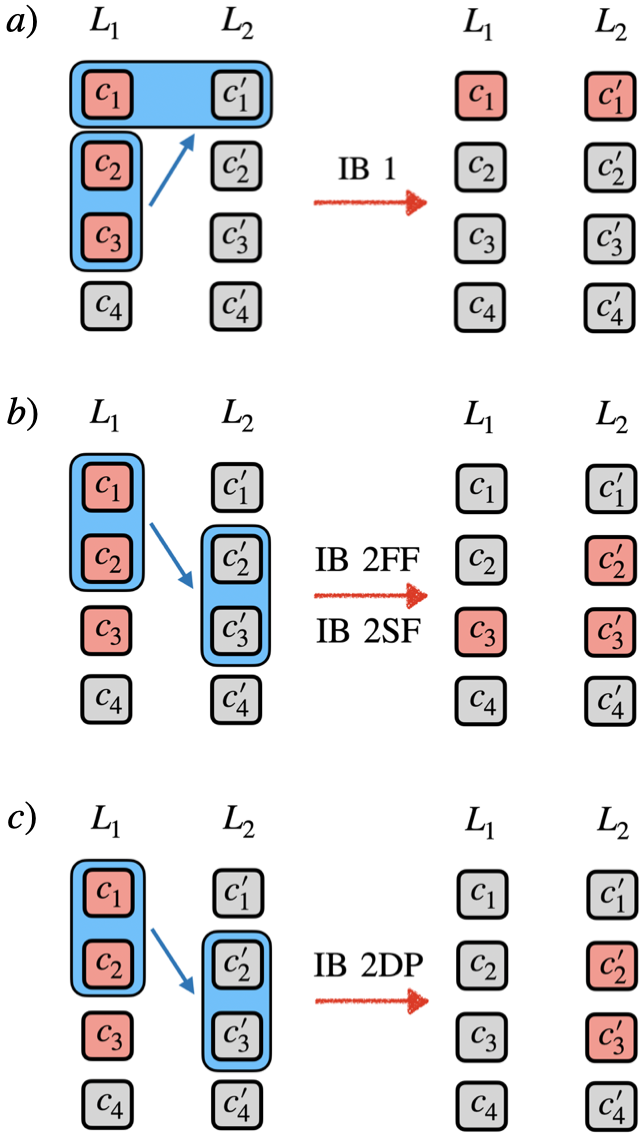}
\caption{An example illustrating random walk steps under different inductive biases, compare with Example \ref{Example:InductiveBiases}. Excited atomic clips are shown in red. The sampled hyperedge is shown in blue. Subfigure a) shows a deliberation step which is only allowed under Inductive Bias \ref{Bias:ManyBody}, because its codomain is in two layers. Also, it shows that an excitation moving into an occupied atomic clip gets discarded. Subfigure b) shows a typical transition under Inductive Biases \ref{Bias:Layered} and 2SF. Notice that it can lead to excitations being spread over several layers. Subfigure c) shows a typical transition under Inductive Bias 2DP. It discards uninvolved excitations, and therefore only contains excitations in the codomain of the hyperedge. }

\label{Figure:InductiveBias}    
\end{figure}

\begin{example}[Example~\ref{Example:InductiveBiases} including Inductive Biases 2SF and 2DP]
Consider a simple 2-layer setting, with 4 atomic clips in each layer, see Figure \ref{Figure:InductiveBias}: $V=L_1 \cup L_2$, with $L_1 = \{c_1 , c_2, c_3, c_4\}$ and $L_2 = \{c'_1, c'_2, c'_3, c'_4\}$. We only consider h-values with the same number of incoming and outgoing excitation numbers, and let no more than two excitations interact. That means $IO = \{(1,1), (2,2)\}$. Our current excitation configuration is $\{ c_1, c_2, c_3 \}$, meaning that we currently have an excitation in each of the atomic clips $c_1$, $c_2$, and $c_3$. 

With the weakest of the inductive biases, i.e. Inductive Bias \ref{Bias:ManyBody}, and choosing $E^{(i,o)} = E_{\mathrm{all}}^{(i,o)}$, our list $\mathcal H_{\mathrm{relevant}}$ of currently relevant h-values is:
\begin{enumerate}
    \item $h_{(2,2)}(\{c_m, c_n\}, \{c'_j, c'_k\})$ such that $j,k \in \{1,2,3,4\}$, $j < k$ and $m,n \in \{1,2,3\}$, $m < n$ 
    \item $h_{(2,2)}(\{c_m, c_n\}, \{c_j, c_k\})$ such that $j,k \in \{1,2,3,4\}$, $j < k$ and $m,n \in \{1,2,3\}$, $m < n$, and $\{j,k\} \ne \{m,n \}$
    \item $h_{(2,2)}(\{c_m, c_n\}, \{c_j, c'_k\})$ such that $j,k \in \{1,2,3,4\}$ and $m,n \in \{1,2,3\}$, $m < n$
    \item $h_{(1,1)}(c_m, c'_j)$ such that $j \in \{1,2,3,4\}$ and $m \in \{1,2,3\}$
    \item $h_{(1,1)}(c_m, c_j)$ such that $j \in \{1,2,3,4\}$, and $m \in \{1,2,3\}$, and $j \ne m$
\end{enumerate}
This list gets turned into probabilities, in our example by applying the softmax-function to the full list. Say, we sample $h_{(2,2)}(\{c_2, c_3\}, \{c_1, c'_1\})$ and apply it to our current configuration $\{c_1, c_2, c_3\}$. First, we remove the excitations in $c_2$ and $c_3$, giving us the configuration $\{c_1\}$. Next, we put excitations into $c_1$ and $c'_1$. However, $c_1$ already carries an excitation. We just keep this excitation as it is. So our next excitation configuration is $\{c_1, c'_1\}$. Note that our rule for dealing with already occupied atomic clips led to a reduction in the total number of excitations.

Our layered Inductive Biases \ref{Bias:Layered} and 2SF differ from the previous situation in that the relevant many-body h-values are only items 1 and 4 from the numbered list above. Now, say that we sampled $h_{(2,2)}(\{c_1, c_2\}, \{c'_2, c'_3 \})$ and apply it to our current configuration $\{c_1, c_2, c_3\}$. First, we remove the excitations in $c_1$ and $c_2$, giving us the configuration $\{c_3\}$. Next, we insert excitations in $c'_2, c'_3$, giving us the full next excitation configuration $\{c'_2, c'_3, c_3\}$. We observe that while the feed-forward condition forces all excitations that move to move one layer forward, it allows excitations to stay behind in their old atomic clip in the old layer. Consider now an additional layer $L_3$. While Inductive Bias \ref{Bias:Layered} allows us to continue with any transition $C_{\mathrm{in}} \rightarrow C_{\mathrm{out}}$ that has $C_{\mathrm{in}} \subset \{c'_2, c'_3\}$ or $C_{\mathrm{in}} = \{c_3\}$, Inductive Bias 2SF forces us to remove $c_3$ first. For Inductive Bias 2SF, the relevant many-body h-values are therefore just those of the form $h_{(1,1)}(c_3, c'_j)$ with $j=1,\dots 4$.  

Now, consider Inductive Bias 2DP acting on $\{c_1, c_2, c_3\}$. While it has the same list of relevant $h$-values as Inductive Bias \ref{Bias:Layered}, it applies the transitions another way. Again, assume that we sampled $h_{(2,2)}(\{c_1, c_2\}, \{c'_2, c'_3 \})$ and apply it to $\{c_1, c_2, c_3\}$. Again, we remove the ingoing excitations $c_1$ and $c_2$, giving us $\{c_3\}$. Next, we insert excitations in $c'_2$ and $c'_3$, giving us $\{c'_2, c'_3, c_3\}$. Furthermore, $c_3$ is neither an ingoing nor an outgoing atomic clip of $h_{(2,2)}(\{c_1, c_2\}, \{c'_2, c'_3 \})$, so we discard its excitation. This gives us the full next excitation configuration $\{c'_2, c'_3\}$. As we see, Inductive Bias 2DP enforces that after a transition, all excitations are found in the same layer.
\end{example}

When modeling agents with Inductive Bias 2SF, there is one important consequence to keep in mind: In human decision-making, a common theme is to write down some intermediate results, and only use them much later when they are deemed relevant. An example would be the derivation of several independent lemmas, all of which get used in proving a theorem. Since in Inductive Bias 2SF the shallowest excitations are removed first, one should introduce copies of their atomic clips in deeper layers to not lose the knowledge/concepts they represent in later steps. 

Inductive Bias 2DP forces all excitations to move forward, and discards those that failed to do so. This models agents with a short attention span who forget everything that is not immediately relevant.

\section{The Relation Between Hypergraphs} \label{Appendix:RelationHypergraphs}
In the main text, we have introduced two different concepts of hypergraphs. The (weighted) ECM which uses standard $h$-values $h$, and the many-body hypergraph which uses many-body $h$-values $h_{(i,o)}$. In this appendix, we work out the relation between the two by constructing the standard $h$-values from the $h_{(i,o)}$ for Inductive Bias \ref{Bias:ManyBody}, under the standard probability assignment.

Consider any two excitation configurations $C_{\mathrm{in}} = \{c_{j_1}, \dots c_{j_i} \}$ and $C_{\mathrm{out}} = \{c_{k_1}, \dots c_{k_o} \}$. Given the weighted many-body hypergraph, we set

\begin{align}
    &h(C_{\mathrm{in}}, C_{\mathrm{out}}) = \label{Equation:RelationHypergraphs}\\
    &\sum_{ \substack{(i,o) \in IO,\ (C_{\mathrm{in}}^{(i)} \rightarrow C_{\mathrm{out}}^{(o)}) \in E^{(i,o)} \\ \text{such that } C_{\mathrm{in}}^{(i)} \subset C_{\mathrm{in}},\ C_{\mathrm{out}} = (C_{\mathrm{in}} \setminus C_{\mathrm{in}}^{(i)}) \cup C_{\mathrm{out}}^{(o)} }} h_{(i,o)}(C_{\mathrm{in}}^{(i)}, C_{\mathrm{out}}^{(o)}) \nonumber
\end{align}

for each $e = (C_{\mathrm{in}} \rightarrow C_{\mathrm{out}})$ that has at least one summand in Eq. \eqref{Equation:RelationHypergraphs}. The set $E$ of standard hyperedges is then the union of all such $e$. If instead for a $(C_{\mathrm{in}} \rightarrow C_{\mathrm{out}})$ the sum has no summands, it is not an allowed hyperedge, and we do not associate a standard $h$-value with it. 

We remind that $|C_{\mathrm{out}}^{(o)}| = o$, $|C_{\mathrm{in}}^{(i)}| = i$, and $C_{\mathrm{out}}^{(o)} \ne C_{\mathrm{in}}^{(i)}$ are conditions already required by elements of $E^{(i,o)}$.

Now, let us prove that Eq. \eqref{Equation:RelationHypergraphs} leads to the same probabilities for transitions between excitation configurations for both $h$ and $h_{(i,o)}$. For this proof, we will have to assume that the standard PS probability assignment is used, i.e. $p_j = \frac{h_j}{ \sum_k h_k}$, both for the ECM and the many-body hypergraph. We will use the transition sampling rules of Inductive Bias \ref{Bias:ManyBody}, and show that they correspond to the standard transition sampling rule using the $h$-values of Eq. \eqref{Equation:RelationHypergraphs}. In other words, under Inductive Bias \ref{Bias:ManyBody}, the standard probability assignment, and Equation \eqref{Equation:RelationHypergraphs}, we show that the standard MEPS agent and the many-body MEPS agent have the same probabilities for all random walks. This means they are equivalent during inference.

Within Inductive Bias \ref{Bias:ManyBody}, we sample a transition using the $h_{(i,o)} (C_{\mathrm{in}}^{(i)}, C_{\mathrm{out}}^{(o)})$. Applying the sampled transition to an  excitation configuration $C_{\mathrm{in}}$ gives the next excitation configuration $C_{\mathrm{out}} := (C_{\mathrm{in}} \setminus C_{\mathrm{in}}^{(i)}) \cup C_{\mathrm{out}}^{(o)}$. 

We group different $h_{(i,o)} (C_{\mathrm{in}}^{(i)}, C_{\mathrm{out}}^{(o)})$ together which result in the same $C_{\mathrm{out}}$. Getting a particular $C_{\mathrm{out}}$ has the probability
\begin{align}
    &p(C_{\mathrm{out}} | C_{\mathrm{in}}) = \label{Equation:RelationProbsGraphs}\\
    &\sum_{ \substack{(i,o) \in IO,\ (C_{\mathrm{in}}^{(i)} \rightarrow C_{\mathrm{out}}^{(o)}) \in E^{(i,o)} \\ \text{such that } C_{\mathrm{in}}^{(i)} \subset C_{\mathrm{in}},\ C_{\mathrm{out}} = (C_{\mathrm{in}} \setminus C_{\mathrm{in}}^{(i)}) \cup C_{\mathrm{out}}^{(o)} }} p(C_{\mathrm{in}}^{(i)} \rightarrow C_{\mathrm{out}}^{(o)}) \nonumber
\end{align}
The first line in the sum of Eq. \eqref{Equation:RelationProbsGraphs} just says ``only consider transitions that are in the set of allowed transitions $E^{(i,o)}$''. This is the same as the first line in Eq. \eqref{Equation:RelationHypergraphs}, and does not yet take into account which transitions are applicable to the full configuration $C_{\mathrm{in}}$.   

The second line in the sum of Eq. \eqref{Equation:RelationProbsGraphs} expresses the fact that applying $h_{(i,o)} (C_{\mathrm{in}}^{(i)}, C_{\mathrm{out}}^{(o)})$ to $C_{\mathrm{in}}$ within Inductive Bias \ref{Bias:ManyBody} is allowed only if $C_{\mathrm{in}}^{(i)} \subset C_{\mathrm{in}}$, and if so it proceeds by removing the excitations in $C_{\mathrm{in}}^{(i)}$ from $C_{\mathrm{in}}$, and then adds the excitations $C_{\mathrm{out}}^{(o)}$. This process has to result in $C^{\mathrm{out}}$.

Now, we need to relate the probabilities to the $h_{(i,o)}$-values under the standard probability assignment rule. For the normalization, we have to consider all output configurations, giving a normalization 
\begin{align}
    \mathcal N = \sum_{ \substack{(i',o') \in IO,\\ (D_{\mathrm{in}}^{(i')}\rightarrow D_{\mathrm{out}}^{(o')}) \in E^{(i',o')} \\ \text{such that } D_{\mathrm{in}}^{(i')} \subset C_{\mathrm{in}} }} h_{(i',o')}(D_{\mathrm{in}}^{(i')}, D_{\mathrm{out}}^{(o')}) \label{Equation:NormalizationAppendix}
\end{align}
resulting in
\begin{align}
    & p( C_{\mathrm{in}}^{(i)} \rightarrow C_{\mathrm{out}}^{(o)}) = \frac{ h_{(i,o)}(C_{\mathrm{in}}^{(i)}, C_{\mathrm{out}}^{(o)}) } { \mathcal N }  \nonumber
\end{align}
With this, Eq. \eqref{Equation:RelationProbsGraphs} takes the form 
\begin{align}
    &p(C_{\mathrm{out}} | C_{\mathrm{in}}) = \label{Equation:RelationGraphsAlmostThere}\\
    &\sum_{ \substack{(i,o) \in IO,\ (C_{\mathrm{in}}^{(i)} \rightarrow C_{\mathrm{out}}^{(o)}) \in E^{(i,o)} \\ \text{such that } C_{\mathrm{in}}^{(i)} \subset C_{\mathrm{in}},\ C_{\mathrm{out}} = (C_{\mathrm{in}} \setminus C_{\mathrm{in}}^{(i)}) \cup C_{\mathrm{out}}^{(o)} }} \frac{h_{(i,o)}(C_{\mathrm{in}}^{(i)}, C_{\mathrm{out}}^{(o)})}{\mathcal N}\nonumber
\end{align}
Now, under the definition in Eq. \eqref{Equation:RelationHypergraphs}, Eq. \eqref{Equation:RelationGraphsAlmostThere} can be rewritten as $p(C_{\mathrm{out}} | C_{\mathrm{in}}) = \frac{h(C_{\mathrm{in}}, C_{\mathrm{out}})}{\mathcal N}$, which is compatible with the standard probability assignment of the standard ECM. As a last consistency check, we point out that the normalization $\mathcal N$ in Eq. \eqref{Equation:NormalizationAppendix} can be rewritten as:
\begin{align}
    &\mathcal N = \sum_{C_{\mathrm{out}}} \nonumber \\ 
    &\sum_{ \substack{(i',o') \in IO,\ (D_{\mathrm{in}}^{(i')}\rightarrow D_{\mathrm{out}}^{(o')}) \in E^{(i',o')} \\ \text{such that } D_{\mathrm{in}}^{(i')} \subset C_{\mathrm{in}},\ C_{\mathrm{out}} = (C_{\mathrm{in}} \setminus D_{\mathrm{in}}^{(i')}) \cup D_{\mathrm{out}}^{(o')} }} h_{(i',o')}(D_{\mathrm{in}}^{(i')}, D_{\mathrm{out}}^{(o')}) \nonumber \\
    & = \sum_{C_{\mathrm{out}}} h(C_{\mathrm{in}}, C_{\mathrm{out}}) \label{Equation:NormalizationRewritten}
\end{align}
Here, if the condition in the second sum cannot be satisfied, we use the convention that that sum is zero. Eq. \eqref{Equation:NormalizationRewritten} differs from Eq. \eqref{Equation:NormalizationAppendix} merely by explicitly spelling out that applying $h_{(i,o)}$ to a configuration $C_{\mathrm{in}}$ will always have to result in a reachable configuration $C_{\mathrm{out}}$. This concludes the proof.

It is important to emphasize that the standard MEPS agent and the many-body MEPS agent are \textbf{NOT equivalent during learning}. The many-body agent will reinforce $h_{(i,o)}$, which according to Equation \eqref{Equation:RelationHypergraphs} will in general change several $h$, even $h$ for transitions that were not performed in the random walk.

We point out that the proof can be adapted to the Inductive Biases \ref{Bias:Layered}, 2SF and 2DP:
\begin{itemize}
    \item Inductive Bias \ref{Bias:Layered} just restricts the set $E^{(i,o)}$, therefore the same proof applies.

    \item Inductive Bias 2SF puts the additional restriction that shallow excitations must move first. Which excitations are the shallowest only depends on $C_{\mathrm{in}}$, and this requirement can be written as an extra condition into all the sums of the proof. Other than that, the proof stays unchanged.

    \item Inductive Bias 2DP replaces the condition $C_{\mathrm{out}} = (C_{\mathrm{in}} \setminus C_{\mathrm{in}}^{(i)}) \cup C_{\mathrm{out}}^{(o)}$ with $C_{\mathrm{out}} = C^{(o)}_{\mathrm{out}}$. Other than that, the proof is unchanged.

\end{itemize}

\section{Detailed Complexity Results and Proofs} \label{Appendix:Complexity}

So far, our resource estimates do not distinguish between our inductive biases. However, we already hinted at the fact that they have crucial complexity differences with regard to the maximal deliberation time. To see the difference, we derive bounds on the deliberation time. 

At first, we point out that Inductive Bias \ref{Bias:ManyBody} still allows for deliberations to be arbitrarily long: For any attainable excitation configuration $\{ c_{m_1}, \dots , c_{m_x} \}$, if it is possible to combine transitions relevant for this configuration to a cycle leading again to $\{ c_{m_1}, \dots , c_{m_x} \}$, then this cycle can lead to arbitrarily long deliberation times. A simple class of Inductive Bias \ref{Bias:ManyBody} agents that have such cycles is presented in the following proposition:

\begin{proposition}\label{Proposition:Cycles}
    Consider a many-body MEPS agent conforming to Inductive Bias~\ref{Bias:ManyBody}. Assume for all $(i,o) \in IO$ that $E_{\mathrm{all}}^{(i,o)} = E^{(i,o)}$ and $IO = \{ 1, 2, \dots , k\}^{\times 2}$. Then for all $n\in \mathbb N$, there exists a deliberation chain/random walk taking more than $n$ steps.
\end{proposition}
\begin{proof}
     Let $\{ c_{m_1}, \dots , c_{m_x}\}$ be any excitation configuration. Then we can use $h_{(1,1)}(c', c_{m_1})$ and $h_{(1,1)}(c_{m_1}, c')$ for any atomic clip $c'$ to move back and forth between $\{ c_{m_1}, \dots , c_{m_x}\}$ and $\{ c', c_{m_2}, \dots , c_{m_x}\}$ arbitrarily many times.
\end{proof}

One important consequence of our feed-forward conditions is that they explicitly break such cycles. In that regard, we first establish that our Inductive Bias \ref{Bias:Layered} indeed leads to a finite upper bound on the deliberation time:

\begin{proposition}\label{Proposition:UpperBoundArbitraryFirst}
    Assume Inductive Bias \ref{Bias:Layered} for a layered many-body MEPS agent with layers $(L_1,\dots , L_D)$. Then the deliberation time (i.e. the total number of random walk steps) is upper-bounded by $ \prod_{j=1}^{D-1} (|L_j|+1)$.
\end{proposition}

\begin{proof}
    To get to the worst-case scaling, we assume that actions are only coupled out when only the final layer has excitations, not earlier.

    The transitions that remove the least excitations while creating the most (that then have to be removed one by one) are those with $(i,o) = (1, |L_j|)$.

    Therefore, for the worst case, we assume that for all $j = 2, \dots, D$, all $(1,|L_j|) \in IO$. Also, for all $j$ we require that $E^{(1,|L_j|)}_{\mathrm{all}} = E^{(1,|L_j|)}$. 

    We perform a proof by induction in the number of layers $D$. First, we consider the case $D = 2$. Here, there is only one non-final layer. The slowest way to remove the excitations in layer $L_1$ is by removing them one-by-one, which can be done with $(1,|L_2|)$-transitions. This takes $|L_1|$-many steps.

    Now, assume that the claim is true for all layered many-body hypergraphs that have up to $D-1$ layers and that we have $D$ layers. 

    We decompose our agent into two segments: Segment 1 is $(L_1, \ldots, L_{D-1})$, and Segment 2 is $(L_{D-1}, L_D)$. Here, layer $L_{D-1}$ effectively plays the role of the final layer in Segment 1, and we will refer to it as such.

    First, we point out that transitions within Segment 2 do not affect the excitations in the non-final layers of Segment 1. In particular, transitions within Segment 2 do not change the number of steps we need to empty the non-final layers of Segment 1. However, transitions within Segment 1 may fill up layer $L_{D-1}$. 
    
    Since transitions within Segment 2 do not change the number of steps needed for Segment 1, we can empty $L_{D-1}$ in Segment 2 every time new excitations arrive, before continuing with emptying Segment 1. This ensures that every time Segment 1 moves excitations into Segment 2, the atomic clips are empty and no excitations are discarded in layer $L_{D-1}$. If we did not always empty layer $L_{D-1}$ first, it would not affect the number of steps needed to empty Segment 1, but it may reduce the total number of transitions within Segment 2 because moving a new excitation into an already excited atomic clip just discards the excitation. That is one less excitation to remove one-by-one. This argument shows that the worst random walk empties the deepest non-final layer first. 

    An upper bound on the worst case is given by the assumption that every step within Segment 1 fills up all of layer $L_{D-1}$. The slowest method to empty layer $L_{D-1}$ needs $|L_{D-1}|$ steps and can use $(1,|L_D|)$-transitions for doing so. 

    By the induction hypothesis, an upper bound on the number of transitions within Segment 1 emptying Segment 1 is provided by $\prod_{j=1}^{D-2} (|L_j|+1)$. For the upper bound, after each of these transitions, we use $|L_{D-1}|$ Segment 2 transitions to empty layer $L_{D-1}$. So, the upper bound on the total number of steps is:
    \begin{align*}
        & \prod_{j=1}^{D-2} (|L_j|+1) + |L_{D-1}|\cdot \prod_{j=1}^{D-2} (|L_j|+1) \\ 
        =& \prod_{j=1}^{D-1} (|L_j|+1)
    \end{align*}
    The first product on the left-hand side is the upper bound on the transitions within Segment 1, and the second product expresses that for each transition within Segment 1, we also have up to $|L_{D-1}|$ transitions within Segment 2.
\end{proof}

While the upper bound in Proposition \ref{Proposition:UpperBoundArbitraryFirst} is finite, it is exponential in the depth $D$. However, there exist scenarios in which there is also an exponential lower bound on the maximal deliberation time, which is the content of the following proposition.
    
\begin{proposition}\label{Proposition:DeepToShallowIsTight}
    Consider a MEPS agent obeying Inductive Bias \ref{Bias:Layered} with layers $L_1,\dots, L_D$. 
    For any $o \le |L_2|, \dots , |L_D|$ with $o \ge 2$, assume that $(1,o) \in IO$ and let $E^{(1,o)}_{\mathrm{all}} = E^{(1, o)}$. Assume that actions are coupled out only when only the final layer has excitations and that there exist percept excitation configurations that have at least $o$ excitations in layer $L_1$. 

    A lower bound on the maximal deliberation time is then given by $\prod_{j=1}^{D-1} o = o^{D-1}$. 
\end{proposition}

\begin{proof}
    The proof proceeds by induction over the layers, from deep to shallow. For that, we first establish that the cost to remove one excitation from layer $L_{D-2}$ with $(1,o)$-transitions is $o+1$, because removing that excitation in $L_{D-2}$ creates $o$ excitations in $L_{D-1}$, which can be removed with $o$ transitions $(1,o)$. 

    Next, assume that $L_{j+1}$ is the deepest non-final layer with excitations, and that removing an excitation from layer $L_{j+1}$ and emptying all of layers $L_{j+2}$, \dots , $L_{D-1}$ can be done with a sequence of transitions that takes at least $\prod_{k=j+1}^{D-2} o$ steps.

    Let us now consider an excitation in layer $L_j$, and assume layers $L_{j+1},\dots , L_{D-1}$ are empty. Removing it with a $(1,o)$-transition creates $o$ excitations in $L_{j+1}$. To remove each of these new excitations and also to empty the non-final layers ahead, the induction hypothesis claims that there is a random walk that does this with at least $\prod_{k=j+1}^{D-2} o$ steps. We have to do this for all the $o$ excitations, giving us a number of steps $\prod_{k=j}^{D-2} o$.

    Induction therefore shows that removing an excitation in Layer $L_1$ can be done with a random walk that takes at least $\prod_{k=1}^{D-2} o$ steps. If we consider a percept configuration with $o$ excitations in the first layer, we can first empty the deeper layers, and then all the $o$ excitations in layer $L_1$ using at least $o\cdot \prod_{k=1}^{D-2} o$ steps.
\end{proof}

The expression $\prod_{k=j}^{D-1} o $ from the previous proposition shows that in general, an exponential (in $D$) maximal deliberation time is unavoidable. To get this exponential scaling, it is enough to consider examples with $IO = \{(1,2)\}$, despite the fact that this $IO$ only contains small $i$ and $o$. This shows that small values of $i$ and $o$ are not enough to prevent exponentially large random walk lengths.

However, one key point of the proofs is that for all these scenarios, we require $(i,o)$ with $o>i$ such that we can start an ``avalanche'' of excitations. This suggests that, if we choose our inductive bias on $IO$ such that the number of excitations cannot increase, we find a much better upper bound, which is essentially $\mathrm{width} \times \mathrm{depth}^2$:

\begin{proposition}\label{Proposition:NoNewExcitations}
     Consider a layered many-body MEPS agent with layers $(L_1,\dots , L_D)$ pertaining to Inductive Bias \ref{Bias:Layered}. Assume that for all $(i,o) \in IO$ we have $o \le i$. Then the deliberation time is upper-bounded by $(D-1)\cdot \sum_{j=1}^D |L_j|$.
\end{proposition}
\begin{proof}
    By definition of Inductive Bias \ref{Bias:Layered} (modifying Inductive Bias \ref{Bias:ManyBody}) all deliberations end when all excitations are in layer $L_D$, or earlier. The number of excitations cannot increase. Each time a transition is sampled, at least one excitation is removed or moved to the next layer. Moving an excitation to the final layer takes at most $D-1$ steps. There are at most $\sum_{j=1}^D |L_j|$ excitations, each of which requires no more than $D-1$ steps to be removed or moved to the final layer. In total, we require no more than $(D-1)\times \sum_{j=1}^D |L_j|$-steps.
\end{proof}

Furthermore, the proofs of the exponential lower bound on the worst-case deliberation times required that one first has to remove excitations from the deepest layers. This observation motivated us to introduce Inductive Bias 2SF, which enforces the opposite order, i.e. that shallow excitations need to be removed first. We now argue that this change leads to upper bounds in the maximal deliberation time that are linear in the width and depth of the layered ECM:

\begin{proposition} \label{Proposition:UpperBoundLastFirst}
    Consider a many-body MEPS agent with layers $L_1, \dots , L_D$ obeying Inductive Bias 2SF. Then the maximal deliberation time is upper-bounded by $\sum_{j=1}^{D-1} |L_j| $
\end{proposition}

\begin{proof}
    Again, all random walks end at the latest when all excitations arrive in the last layer. According to Inductive Bias 2SF, we empty the layers going from shallow to deep. 
    Now, assume that layer $L_j$ is the shallowest layer that has excitations. At most $|L_j|$ transitions are needed to empty this layer (each transition removes at least one excitation from $L_j$), and only excitations in deeper layers can be created. In total, this gives $\sum_{j=1}^{D-1} |L_j| $ deliberation steps.
\end{proof}

As we see, we reduced the complexity from exponential to linear scaling with the depth $D$, getting a scaling that is essentially $\mathrm{width} \times \mathrm{depth}$.

Now we consider the harshest of our inductive biases, i.e. Inductive Bias 2DP. For it, we find:

\begin{proposition}
    Consider a layered many-body MEPS agent with layers $(L_1,\dots L_D)$ obeying Inductive Bias 2DP. Then the maximal number of deliberation steps is upper-bounded by $(D-1)$.    
\end{proposition}
\begin{proof}
    Follows from the fact that excitations can only move forward, and the discarding of excitations that are left behind.
\end{proof}

Inductive Bias 2DP can be interpreted as describing agents who only remember the conclusions of their latest thought, and forget all the thoughts that happened before in the deliberation. We see that while such agents are very restricted in their short-term memory, they also have the lowest guaranteed deliberation time, scaling only with the depth but not with the width of the ECM. 

In Proposition \ref{Proposition:Bias1} we showed that the total number of trainable parameters is polynomial in the number of atomic clips. As a consequence, also the number of transitions that we need to consider in each deliberation step scales polynomially. However, do we need to consider all trainable parameters for sampling a transition? The following proposition gives a tighter bound:
\begin{proposition}
    Consider an ECM obeying Inductive Bias \ref{Bias:ManyBody}, \ref{Bias:Layered}, 2SF, or 2DP. Define $\max I := \max \{ i \ | \ \exists o: (i,o) \in IO \}$ and $\max O := \max \{ o \ | \ \exists i: (i,o) \in IO \}$. Then:
        
    At each deliberation/ random walk step on a configuration $\{c_{m_1}, \dots , c_{m_x}\}$ with $x \ge 2$, the number of relevant $h$-values is:
    
    $\mathcal O \big(  \min \big\{ 2^{|V|} ,  |V|^{\max O} \big\} \cdot \min \big\{ 2^x ,  x^{\max I } \big\} \big) $ 
\end{proposition}

\begin{proof}
    Let us bound the number of $C_I, C_O \in \mathcal P(V)$ for the h-values $h_{(i,o)}(C_I, C_O)$ relevant for configuration $\{c_{m_1}, \dots , c_{m_x}\}$. We notice that $C_I$ must be a subset of $\{c_{m_1}, \dots , c_{m_x}\}$. There are at most $2^x$ choices for such subsets. A different bound taking into account the fact that $|C_I| \le \min \{x , \max I\}$ is obtained as follows. For each $i$ such that there is an $o$ with $(i,o) \in IO$, there are $\begin{pmatrix} x \\ i \end{pmatrix} \le x^i$ choices. In total, the number of choices for $C_I$ is upper-bounded by $\sum_{i = 1}^{\min \{ \max I , x \}} x^i \le \frac{x^{\min \{ \max I , x \} + 1}-1}{x-1} \le \frac{x^{\min \{ \max I , x \}+1} } {x-1} \le  \frac{x^{\min \{ \max I , x \}+1}} { \frac{1}{2} x} = 2 x^{\min \{ \max I , x \}}$. Since we already have an upper bound $2^x$ and $x \ge 2$, we can leave out the case $x$ in the minimum of the exponent. 
    
    The bounds for the number of $C_O$ are established in the same way as we just did for the bound $C_I$. Also here, we point out that Inductive Bias \ref{Bias:ManyBody} has all h-values that the Inductive Biases 2 have, and usually more.
\end{proof}

Compared to Proposition \ref{Proposition:Bias1}, this result can provide significant benefits if the number $x$ of excitations in the current configuration is small (note that $x \le |V|$ always), but the number of atomic clips $|V|$ is very large.

\end{document}